\documentclass[twoside,11pt]{article}

\usepackage{jmlr2e}

\usepackage{latexsym,amssymb,amsmath,ifthen,afterpage, amsfonts, mathtools}

\usepackage{enumerate}
\usepackage[utf8]{inputenc} 
\usepackage[T1]{fontenc}    
\usepackage{lmodern}
\usepackage{url}            
\usepackage{booktabs}       
\usepackage{nicefrac}       
\usepackage{microtype}      
\usepackage{algorithm}
\usepackage{algorithmic}
\usepackage{balance}
\usepackage{esvect}
\usepackage{subfigure}

\usepackage{tikz} 
\usepackage{pgfplots}
\usepackage{pgfplotstable}
\usetikzlibrary{patterns,arrows,arrows.meta}
\pgfplotsset{compat = newest}
\pgfkeys{/pgf/number format/set thousands separator = }
\tikzset{
    myarrow/.style={line width=0.45mm, -{Triangle[length=2.0mm,width=2.0mm]}}
}

\definecolor{chocolate1}{RGB}{1,100,5}
\definecolor{chocolate2}{RGB}{145,65,12}

\newcommand{\Fig}[1]{Fig.~\ref{#1}}

\newcommand{\R}{\mathbb{R}}

\newcommand{\ZZ}{\mathbb{Z}}

\newcommand{\calA}{\mathcal{A}}

\newcommand{\calU}{\mathcal{U}}

\newcommand{\calN}{\mathcal{N}}

\newcommand{\EE}{\mathcal{E}}
\newcommand{\X}{{\bf X}}
\newcommand{\x}{{\bf x}}
\newcommand{\Y}{{\bf Y}}
\newcommand{\y}{{\bf y}}

\newcommand{\E}{\operatorname{E}}

\newcommand{\e}{{e'}}

\newcommand{\G}{\mathcal{G}}
\newcommand{\VV}{\mathcal{V}}
\newcommand{\HH}{\mathcal{H}}

\newcommand{\myfrac}[3][0.7pt]{\genfrac{}{}{}{}{\raisebox{#1}{$#2$}}{\raisebox{-#1}{$#3$}}}
\newcommand{\myfracc}[3][2pt]{\genfrac{}{}{}{}{\raisebox{#1}{$#2$}}{\raisebox{-#1}{$#3$}}}

\DeclareMathOperator{\csch}{csch}

\jmlrheading{23}{2022}{1-36}{2/21; Revised
6/22}{8/22}{21-0200}{Mehdi Molkaraie}

\ShortHeadings{Mappings for Marginal Probabilities}{Molkaraie}
\firstpageno{1}

\begin{document}

\title{Mappings for Marginal Probabilities\\ with Applications to Models in Statistical Physics}

\author{\name Mehdi Molkaraie \email mehdi.molkaraie@alumni.ethz.ch \\
       \addr Department of Statistical Sciences\\
       University of Toronto\\
       Toronto ON M5G 1Z5\\ 
       Canada
       }

\editor{Mohammad Emtiyaz Khan}

\maketitle

\begin{abstract}
We present local mappings that relate the marginal probabilities of a global probability mass function represented by its
primal normal factor graph to the corresponding marginal probabilities in its dual normal factor 
graph. The mapping is based on the Fourier transform
of the local factors of the models. Details of the mapping are provided for the 
Ising model, where it is proved that the local extrema of the fixed points are attained at the phase
transition of the two-dimensional nearest-neighbor Ising model.
The results are further extended to the Potts model, to the clock model, and to Gaussian Markov random fields. 
By employing the mapping, we can transform 
simultaneously all the estimated marginal probabilities from the dual domain to the primal domain (and vice versa), which 
is advantageous if estimating the marginals 
can be carried out more efficiently in the dual domain.
An example of particular significance is the ferromagnetic Ising model in a positive 
external magnetic field. For this model, there exists a rapidly mixing Markov chain (called the subgraphs--world process) 
to generate configurations in the dual normal factor graph of the model.
Our numerical experiments illustrate that the proposed procedure can provide more accurate 
estimates of marginal probabilities of a global probability mass function in various settings. 
\end{abstract}


\section{Introduction}

One of the main objectives in any probabilistic inference problem is
computing local marginal probabilities of a global multivariate probability mass function (PMF). In general, such a 
computation requires a summation with an exponential number of terms, which makes its exact computation 
intractable \citep[Chapter 2]{wainwright2008graphical}. 
See also~\citep{dagum1993approximating} for the hardness results on approximating conditional 
probabilities in the context of belief networks. 

In this paper, our approach for estimating the local marginal probabilities of a global PMF hinges on the notions of normal 
realization, in which there is an edge for every variable~\citep{Forney:01}, 
normal factor graph (NFG), and the NFG duality theorem~\citep{AY:2011}. 
A succinct discussion on factor graphs and NFGs is provided in Appendix~A.

In our analysis, we mainly focus on well-known models in Statistical Physics.
Specifically, we consider binary models 
with symmetric pairwise nearest-neighbour interactions (e.g., the Ising model). We will also consider
extensions of our results to non-binary models (e.g., the $q$-state Potts model and the 
clock model), and briefly discuss generalizations to continuous models (e.g., Gaussian Markov random fields). 
These models are widely used in machine learning \citep{murphy:2012}, image 
processing \citep{winkler1995, mcgrory2009variational, liu2022deep}, and 
spatial statistics~\citep{besag1993spatial}. Furthermore, the well-known models in Statistical Physics can be used as the basis 
to analyze less conventional and more complex models. 
Connections between Statistical Physics and Machine Learning run deep. For a comprehensive review, see~\citep{carleo2019machine}
and references therein. 

The NFG duality theorem states that the partition 
function of a primal
NFG and the partition function of its dual NFG are equal up to some 
known scale factor.  
It has been demonstrated that, in the low-temperature regime, Monte 
Carlo methods for estimating the partition function converge faster in the dual NFG than in the primal 
NFG of the two-dimensional (2D) 
Ising model~\citep{MoLo:ISIT2013} and of the 2D 
Potts model \citep{ AY:2014, MoGo:2018}.
This work extends the previous results to marginal probabilities of a primal NFG and its dual.

A preliminary version of this work was presented at \citep{molkaraie2020marginal}, where 
the NFG duality theorem was employed to establish a connection between the marginal probabilities of a global PMF associated with a primal NFG and
their corresponding marginals in its dual NFG via local mappings.
In this paper, we give a more detailed analysis on the proposed mappings -- specially for non-binary and for continuous models. New numerical 
experiments are also provided for continuous graphical models.

The mapping can be expressed in terms of the discrete Fourier transform (DFT)
of the local factors of the graphical models. 
Furthermore, the mappings are independent of the size of the model, of the topology of the graphical model, and of any assumptions on 
the parameters of the model (e.g., the coupling parameters).

The proposed method is
practically beneficial if estimating marginal probabilities
can be carried out more efficiently in the dual domain than in the primal domain.
Indeed, there is a rapidly mixing Markov chain (called the subgraphs-world process) 
to generate configurations in the dual NFG of ferromagnetic Ising models with arbitrary topology and in a positive 
external field. The marginal probabilities in the dual domain can thus be estimated from such configurations. 
For more details see~\citep{JS:93}.

In a related topic and in the context of Coding Theory, it has been demonstrated that, for a class of binary symmetric channels, the correlation between code 
bits of low-density parity-check (LDPC) codes and low-density generator-matrix (LDGM) codes decays exponentially 
fast in the high-noise and the low-noise regimes, respectively. 
For more details see~\citep{kudekar2011decay} and references therein.
LDGM codes can be viewed as the dual of LDPC codes, and vice versa~\citep[Chapter 7]{RU:08}.

As an example, let us consider a 2D homogeneous Ising model with periodic boundary conditions. In 
this model, all marginal probabilities can be expressed as the ratio of two partition functions. In the high-temperature regime,
the ratio can be estimated efficiently in the primal NFG, and at low temperatures, the ratio can be estimated efficiently in the 
dual NFG of the model~\citep{MoLo:ISIT2013, AY:2014, MoGo:2018}. 
In more general
settings (e.g., in non-homogeneous models), each marginal probability needs to be estimated \emph{separately} as the ratio of two partition functions. 
However, the mappings of this paper allow a \emph{simultaneous} transformation of 
estimated marginal probabilities from one domain to the other. In the dual domain, 
the marginal probabilities can be estimated via
variational inference algorithms (e.g., the belief propagation (BP) and the tree expectation 
propagation (TEP) algorithms), via Markov chain Monte Carlo methods, or in the special case of the ferromagnetic 
Ising models via the subgraphs-world process.


The paper is organized as follows. In Sections~\ref{sec:Ising} and~\ref{sec:PrimalNFG} 
we describe our models in the primal domain and discuss their graphical representations in terms of NFGs. 
Section~\ref{sec:Dual} discusses the dual NFG, the NFG duality theorem, and the high-temperature series expansion
of the partition function of the Ising model. In Section~\ref{sec:Marginals}, we derive mappings that relate the edge/vertex marginal probabilities 
in the primal NFG to the corresponding edge/vertex marginal probabilities in the dual NFG.
For binary models, details of the mapping, the fixed points of the mapping, and its connections to local 
magnetization are discussed in
Section~\ref{sec:IsingModelMapping}. Generalizations of the mapping and its fixed points to 
non-binary models 
are discussed in Sections~\ref{sec:GenP}. 
Numerical experiments for 2D and fully-connected graphical models are reported in Section~\ref{sec:NumExp}. 
Finally, section~\ref{sec:Continuous} briefly discusses extensions of the mappings to continuous (Gaussian) models.
Appendix~A is a compact tutorial on factor graphs, NFGs, and the NFG duality. 
In Appendix~B, the mapping 
is employed as an alternative method to compute the marginal probabilities of the 1D Ising and Potts models. 



\section{The Primal Model}
\label{sec:Ising}

Let $\G = (\VV, \EE)$ be a  finite, simple, connected, and undirected graph 
with $|\VV|$ vertices (sites) and $|\EE|$ edges (bonds). 
Suppose random variables $\X = (X_v, v \in \VV)$ and $\Y = (Y_e, e \in \EE)$ are indexed by the 
elements of $\VV$ and $\EE$, respectively. We refer to a 
vertex by $X_v$ (i.e., the random variable associated with $v$) or for brevity by its index $v$. Similarly, we may refer to the edge
that connects $X_{k}$ and $X_{\ell}$ by $(k,\ell)$, by $Y_e$, or simply by its index $e$.

Each variable 
takes values in a finite alphabet $\calA$, where $\calA = \ZZ/q\ZZ$, i.e., the ring of 
integers modulo $q$, for some fixed integer $q \ge 2$. However, we may view 
$\calA$ as any finite abelian group with respect to addition. An assignment of values to $\X$ (to $\Y$) is denoted by $\x$ (by $\y$) and is called a 
\emph{configuration}. Here, $\x \in \calA^{|\VV|}$ and $\y \in \calA^{|\EE|}$ are column vectors of length $|\VV|$ and $|\EE|$, respectively.

For a subset $\mathcal{I}$ of $\VV$ we let $\X_\mathcal{I} = (X_i, i \in \mathcal{I})$. 
We further define two indicator functions: the equality indicator function, which is defined as follows
\begin{equation} 
\label{eqn:Definition1}
\delta_{=}(\x_\mathcal{I}) =  \left\{ \begin{array}{ll}
    1, & \text{if $x_1 = x_2 = \ldots = x_{|\mathcal{I}|}$} \\
    0, & \text{otherwise,}
  \end{array} \right.
\end{equation}
and the zero-sum indicator function, which is defined as
\begin{equation} 
\label{eqn:Definition2}
\delta_{+}(\x_\mathcal{I}) =  \left\{ \begin{array}{ll}
    1, & \text{if $x_1 + x_2 + \ldots + x_{|\mathcal{I}|} = 0$} \\
    0, & \text{otherwise.}
  \end{array} \right.
\end{equation}
For $|\mathcal{I}|=1$, both (\ref{eqn:Definition1}) and (\ref{eqn:Definition2}) are equivalent to the Kronecker delta 
function, denoted by $\delta(\cdot)$. Notice that $\delta_{=}(x_k, x_\ell)$ and $\delta_{+}(x_k, x_\ell)$ are both
equal to $\delta(x_k - x_\ell)$ when $\calA = \ZZ/2\ZZ$ (i.e., the binary field).

In this paper, we assume that variables have pairwise interactions. Two variables interact 
if their corresponding vertices are connected by an edge in $\G$. 
Following~\cite{Forney:18}, we suppose $\pmb{M}$ is an oriented incidence 
matrix of $\G$ with size $|\EE| \times |\VV|$. To construct $\pmb{M}$, we first give each edge an arbitrary orientation. We then set the entry
$\pmb{M}_{e,v} = +1$ if $Y_e$ leaves $X_v$, the entry $\pmb{M}_{e,v} = -1$ if $Y_e$ enters 
$X_v$, and the entry $\pmb{M}_{e,v} = 0$ otherwise. Therefore, each row of $\pmb{M}$ has exactly
two nonzero entries, and the number of nonzero entries in the $v$-th column of $\pmb{M}$ is
equal to the degree of $X_v$. Since $\G$ is connected, the rank of $\pmb{M}$ is $|\VV|-1$.

In this framework
\begin{equation} 
\label{eqn:Y}
\y\colon \calA^{|\VV|} \rightarrow \calA^{|\EE|}, \quad  \x \mapsto \pmb{M}\x
\end{equation}
Each entry $y_e$ of $\y(\x) = \pmb{M}\x$ can thus be expressed as the difference between 
$x_{k}$ and $x_{\ell}$, where $X_{k}$ and $X_{\ell}$ are connected by $Y_e$. In the sequel, 
we drop the argument $\x$, and use $\y$ instead of $\y(\x)$ when there is no ambiguity.

In the primal domain, we assume that the probability of a configuration $\x \in \calA^{|\VV|}$  is given by
the following PMF
\begin{equation} 
\label{eqn:ProbP}
\mathrm{\pi}_\text{p} (\x) = \frac{1}{Z_\text{p}}\prod_{e \in \EE} \psi_e(y_e)\prod_{v \in \VV} \phi_v(x_v)
\end{equation}
for some set of \emph{edge-weighing factors}
$\{\psi_{e} \colon \calA \rightarrow \R_{\ge 0}, e \in \EE\}$ and \emph{vertex-weighing factors}
$\{\phi_{v} \colon \calA \rightarrow \R_{\ge 0}, v \in \VV\}$. 
The normalization constant $Z_\text{p}$ 
is the \emph{partition function}, which can be computed as
\begin{equation} 
\label{eqn:ZP}
Z_\text{p} = \sum_{\x \in \calA^{|\VV|}}\prod_{e \in \EE} \psi_e(y_e)\prod_{v \in \VV} \phi_v(x_v),
\end{equation}
where the sum runs over all the vectors in the configuration space $\calA^{|\VV|}$.

In (\ref{eqn:ProbP}) we have assumed that each edge-weighing factor $\psi_{e}(\cdot)$ is only a function of the edge 
configuration $y_e$. In the sequel, we show that many important models in Statistical Physics
(e.g., the Ising model and the $q$-state Potts model) can be 
easily represented in this framework. In Section~\ref{sec:Continuous}, we will show that a Gaussian Markov random field with the 
thin-membrane prior can also be expressed by~(\ref{eqn:ProbP}).


\subsection{Examples from Statistical Physics}

In the Ising and Potts models variables (particles) are associated with the vertices of $\G$, and two
particles interact if they are joined by and edge. A real coupling parameter is associated with 
each interacting pair of particles.
The model is called \emph{ferromagnetic} if coupling parameters are nonnegative. 
If coupling parameters are constant, the model is called \emph{homogeneous}.
Moreover, particles may be under the influence of an external 
magnetic field.
The external
field is usually assumed to be constant. However, in more general settings, the field may 
vary from site to site.

The energy of a configuration $\x \in \calA^{|\VV|}$ is given by the Hamiltonian $\HH(\x)$. 
The probability of a configuration $\x$ is given by the Gibbs-Boltzmann distribution 
$\pi \propto \mathrm{e}^{-\beta\HH(\x)}$, where $\beta$ is proportional to the inverse temperature~\citep{Huang:87}.
We sometimes denote exp(1) by $\mathrm{e}$ to distinguish it from the variable $e$ that we use to denote the edges of $\G$.

\begin{itemize}

\item {\bf The Ising model.} In an Ising model $\calA = \ZZ/2\ZZ$, and each particle may be in one of the two states (e.g., spin up or spin down). 
The Hamiltonian of the model is given by\footnote{In the bipolar case (i.e., when $\calA = \{-1,+1\}$), 
the Hamiltonian of the Ising model has the more familiar form $\HH(\x) 
= -\sum_{(k,\ell) \in \EE} J_{k,\ell}x_k x_\ell - \sum_{v \in \VV}H_{v}x_v$. For more details, see~\cite[Chapter 1]{Baxter:07}.}
\begin{equation} 
\label{eqn:HamiltonianIsing}
\HH(\x) = -\sum_{(k, \ell) \in \EE}J_{k,\ell}\left(2\delta_{=}(x_k, x_\ell)-1\right) 
- \sum_{v \in \VV}H_v\left(2\delta(x_v)-1\right)
\end{equation}

Here $J_{k,\ell}$ and $H_v$ denote the coupling parameter associated with the interacting
pair $(X_k, X_\ell)$ and 
the external field at site $v$, respectively. 

Starting from the Gibbs-Boltzmann distribution, it is 
straightforward to verify that the model can be represented by the PMF~(\ref{eqn:ProbP}), 
in which the edge-weighing factors are
\begin{equation} 
\label{eqn:IsingPot1}
\psi_e(y_e) = \left\{ \begin{array}{ll}
    \textrm{e}^{\beta J_e}, & \text{if $y_e = 0$} \\
    \textrm{e}^{-\beta J_e}, & \text{if $y_e = 1,$}
  \end{array} \right.
\end{equation}
where $\y = \pmb{M}\x$ and $J_e$ is the coupling parameter associated with $e = (k,\ell)$, and the 
vertex-weighing factors are given by
\begin{equation} 
\label{eqn:IsingPot2}
\phi_v(x_v) = \left\{ \begin{array}{ll}
    \textrm{e}^{\beta H_v}, & \text{if $x_v = 0$} \\
    \textrm{e}^{-\beta H_v}, & \text{if $x_v = 1$}
  \end{array} \right.
\end{equation}

\item {\bf The $q$-state Potts model.} In the Potts model 
each variable represents the $q$ possible states of a particle, which takes values in $\calA = \ZZ/q\ZZ$.
The Hamiltonian of the model reads
\begin{equation} 
\label{eqn:HamiltonianPotts}
\HH(\x) = -\sum_{(k, \ell) \in \EE}J_{k,\ell}\delta_{=}(x_k, x_\ell) - \sum_{v \in \VV}H_v\delta(x_v)
\end{equation}
Here, following~\cite[Chapter 1]{Nishimori:15}, we have assumed that the external field affects the variable $x_v$ only if $x_v = 0$.

Again, from the Gibbs-Boltzmann distribution, it is easy to show that the model can be represented by~(\ref{eqn:ProbP}) with
\begin{equation} 
\label{eqn:PottsPot1}
\psi_e(y_e) = \left\{ \begin{array}{ll}
    \textrm{e}^{\beta J_e}, & \text{if $y_e = 0$} \\
    1, & \text{otherwise,}
  \end{array} \right.
\end{equation}
and 
\begin{equation} 
\label{eqn:PottsPot2}
\phi_v(x_v) = \left\{ \begin{array}{ll}
    \textrm{e}^{\beta H_v}, & \text{if $x_v = 0$} \\
    1, & \text{otherwise.}
  \end{array} \right.
\end{equation}

\item {\bf The clock model.} In the absence of an external field, let
\begin{equation} 
\label{eqn:HamiltonianClock}
\HH(\x) = -\sum_{(k, \ell) \in \EE }J_{k,\ell}\cos(2\pi(x_k - x_\ell)/q),
\end{equation}
where $\calA = \ZZ/q\ZZ$. This model is known as the clock model (or the vector 
Potts model), which reduces to the Ising model if $q=2$ and to the three-state Potts model if $q = 3$. 
In the limit of large $q$, the model is equivalent to the XY model.

The clock model can also be expressed by~(\ref{eqn:ProbP}).
E.g., the edge-weighing factors of a four-state clock model with Hamiltonian (\ref{eqn:HamiltonianClock}) are given by
\begin{equation} 
\label{eqn:ClockPot1}
\psi_e(y_e) = \left\{ \begin{array}{ll}
    \textrm{e}^{\beta J_e}, & \text{if $y_e = 0$} \\
       \textrm{e}^{-\beta J_e}, & \text{if $y_e = 2$} \\
    1, & \text{otherwise.}
  \end{array} \right.
\end{equation}

\end{itemize}

For more details, see~\citep{Yeo:92, Baxter:07}.

Next, we will discuss the graphical representation of (\ref{eqn:ProbP}) in terms of NFGs.

\section{The Primal NFG}
\label{sec:PrimalNFG} 

The factorization in~(\ref{eqn:ProbP}) can be represented
by an NFG $\G=(\VV,\EE)$, in which vertices represent the factors 
and edges 
represent the variables.
The edge 
that represents some variable $y_e$ is connected to the vertex 
representing the factor $\psi_e(\cdot)$
if and only if $y_e$ is an argument of $\psi_e(\cdot)$. 
If a variable is involved in more than two factors, it is 
replicated via an equality indicator factor. 
See Appendix~A for more details.

The primal NFG of the 1D Potts model in an external field with periodic boundary conditions is depicted in~\Fig{fig:2DGridMod},
where the big unlabeled boxes (attached to $Y_1, Y_2$, and $Y_3$) represent~(\ref{eqn:PottsPot1}), and 
the small unlabeled boxes (attached to $X_1, X_2$, and $X_3$) represent~(\ref{eqn:PottsPot2}).

In~\Fig{fig:2DGridMod}, boxes labeled ``$=$'' are instances of equality 
indicator factors given by (\ref{eqn:Definition1}), 
which impose the constraint that all their incident variables are equal, e.g., the equality indicator factor 
$\delta_{=}( x_1, x^{\prime}_1, x^{\prime \prime}_1)$ is equal to one if $x_1 = x^{\prime}_1 = x^{\prime \prime}_1$, and
is equal to zero otherwise.
In analogy with 
the inversion bubble connected to the logic NOT, NAND, and NOR gate symbols~\citep{horowitz1989art}, the symbol ``$\circ$'' is
used to indicate a sign inversion.
The boxes labeled ``$+$'' are instances of zero-sum indicator 
factors given by (\ref{eqn:Definition2}). They impose the
constraint that all their incident variables sum 
to zero,  e.g., the zero-sum indicator factor 
$\delta_{+}( y_1, x^{\prime}_1, x^{\prime}_2)$ is equal to one if $y_1 + x^{\prime}_1 + x^{\prime}_2 = 0$, and
evaluates to zero otherwise. (Recall that all arithmetic manipulations 
are modulo $q$.) 

In the primal domain, the NFG of a 1D Ising model with periodic boundary conditions is also shown in~\Fig{fig:2DGridMod}, 
where the big unlabeled boxes represent~(\ref{eqn:IsingPot1}), and 
the small unlabeled boxes represent~(\ref{eqn:IsingPot2}). In the case of the Ising model $\calA = \ZZ/2\ZZ$ (i.e., the binary field), the ``$\circ$'' symbols 
are irrelevant and can be removed from the primal NFG. (Note that in the binary field $x_1 + x_1 = 0$, and therefore $-x_1 = x_1$.)

%
%
%


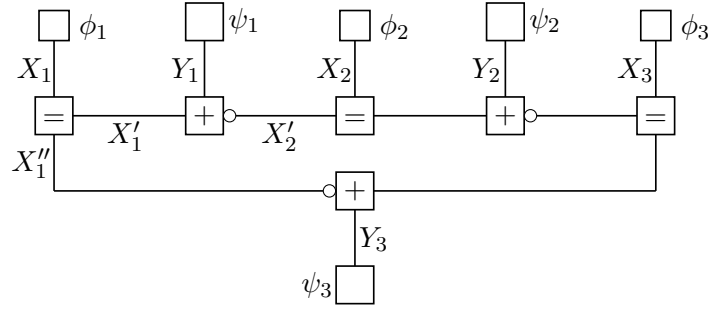
\begin{figure}[t]
  \centering
  \begin{tikzpicture}[scale=1.0]

\draw [line width=0.22mm] (0, 2) rectangle (0.5,2.5);
\draw [line width=0.22mm] (2, 2) rectangle (2.5,2.5);
\draw [line width=0.22mm] (4, 2) rectangle (4.5,2.5);
\draw [line width=0.22mm] (6, 2) rectangle (6.5,2.5);
\draw [line width=0.22mm] (8, 2) rectangle (8.5,2.5);
\draw [line width=0.22mm] (4, 1) rectangle (4.5,1.5);
\draw [line width=0.22mm] (0.5,2.25) -- (2, 2.25);
\draw [line width=0.22mm] (2.5,2.25) -- (4, 2.25);
\draw [line width=0.22mm] (4.5,2.25) -- (6, 2.25);
\draw [line width=0.22mm] (6.5,2.25) -- (8, 2.25);
\draw [line width=0.22mm] (0.25,2.5) -- (0.25, 3.25);
\draw [line width=0.22mm] (2.25,2.5) -- (2.25, 3.25);
\draw [line width=0.22mm] (4.25,2.5) -- (4.25, 3.25);
\draw [line width=0.22mm] (6.25,2.5) -- (6.25, 3.25);
\draw [line width=0.22mm] (8.25,2.5) -- (8.25, 3.25);
\draw [line width=0.22mm] (0.25,1.25) -- (4.0, 1.25);
\draw [line width=0.22mm] (4.5,1.25) -- (8.25, 1.25);
\draw [line width=0.22mm] (4.25,1.0) -- (4.25, 0.25);
\draw [line width=0.22mm] (0.25,2) -- (0.25, 1.25);
\draw [line width=0.22mm] (8.25,2) -- (8.25, 1.25);
\draw [line width=0.22mm] (0.05, 3.25) rectangle (0.45,3.65);
\draw [line width=0.22mm] (2, 3.25) rectangle (2.5,3.75);
\draw [line width=0.22mm] (4.05, 3.25) rectangle (4.45,3.65);
\draw [line width=0.22mm] (6, 3.25) rectangle (6.5,3.75);
\draw [line width=0.22mm] (8.05, 3.25) rectangle (8.45,3.65);
\draw [line width=0.22mm] (4, 0.25) rectangle (4.5,-0.25);
\node at (0.25, 2.20){$=$};
\node at (2.25, 2.25){$+$};
\node at (4.25, 2.20){$=$};
\node at (6.25, 2.25){$+$};
\node at (8.25, 2.20){$=$};
\node at (4.25, 1.25){$+$};
 \draw (0.78,3.43) node{$\phi_1$};
  \draw (4.78,3.43) node{$\phi_2$};
 \draw (8.78,3.43) node{$\phi_3$};
 \draw (2.78,3.52) node{$\psi_1$};
 \draw (6.78,3.52) node{$\psi_2$};
  \draw (3.75,0.0) node{$\psi_3$};
  \draw (0.0,2.86) node{$X_1$};
  \draw (1.2,1.98) node{$X^{\prime}_1$};
    \draw (-0.05,1.62) node{$X^{\prime\prime}_1$};
      \draw (3.25,1.98) node{$X^{\prime}_2$};
  \draw (3.98,2.86) node{$X_2$};
 \draw (7.98,2.86) node{$X_3$};
 \draw (2.0,2.86) node{$Y_1$};
  \draw (6.0,2.86) node{$Y_2$};
    \draw (4.5,0.61) node{$Y_3$};
\draw[fill=white] (2.584,2.25) circle (0.82mm);
\draw[fill=white] (6.584,2.25) circle (0.82mm);
\draw[fill=white] (3.911,1.25) circle (0.82mm);
  \end{tikzpicture}
  \caption{\label{fig:2DGridMod}
The primal NFG of the 1D Potts model in an external field with periodic boundary conditions,
where the big unlabeled boxes represent~(\ref{eqn:PottsPot1}), and 
the small unlabeled boxes represent~(\ref{eqn:PottsPot2}). Boxes labeled ``$=$'' are instances of equality 
indicator factors as in (\ref{eqn:Definition1}),
the boxes labeled ``$+$'' are instances of zero-sum indicator 
factors as in (\ref{eqn:Definition2}), and the symbol ``$\circ$''
denotes a sign inversion.}
\end{figure}


As discussed in Section~\ref{sec:Ising}, we observe that in \Fig{fig:2DGridMod} we can freely choose $\X$ 
and therefrom fully determine $\Y$. More generally, if we 
take $\G$ to be a $d$-dimensional model with pairwise interactions, we can compute each component $Y_e$ of $\Y$ 
from the two components of $\X$ that are incident to the zero-sum indicator 
factor attached to $Y_e$.


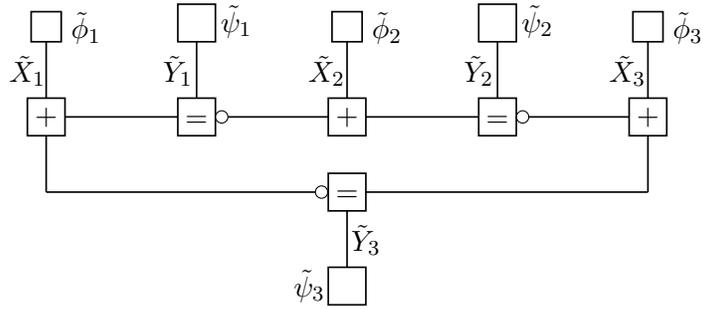
\begin{figure}[t]
  \centering
  \begin{tikzpicture}[scale=1.0]

\draw [line width=0.22mm] (0, 2) rectangle (0.5,2.5);
\draw [line width=0.22mm] (2, 2) rectangle (2.5,2.5);
\draw [line width=0.22mm] (4, 2) rectangle (4.5,2.5);
\draw [line width=0.22mm] (6, 2) rectangle (6.5,2.5);
\draw [line width=0.22mm] (8, 2) rectangle (8.5,2.5);
\draw [line width=0.22mm] (4, 1) rectangle (4.5,1.5);
\draw [line width=0.22mm] (0.5,2.25) -- (2, 2.25);
\draw [line width=0.22mm] (2.5,2.25) -- (4, 2.25);
\draw [line width=0.22mm] (4.5,2.25) -- (6, 2.25);
\draw [line width=0.22mm] (6.5,2.25) -- (8, 2.25);
\draw [line width=0.22mm] (0.25,2.5) -- (0.25, 3.25);
\draw [line width=0.22mm] (2.25,2.5) -- (2.25, 3.25);
\draw [line width=0.22mm] (4.25,2.5) -- (4.25, 3.25);
\draw [line width=0.22mm] (6.25,2.5) -- (6.25, 3.25);
\draw [line width=0.22mm] (8.25,2.5) -- (8.25, 3.25);
\draw [line width=0.22mm] (0.25,1.25) -- (4.0, 1.25);
\draw [line width=0.22mm] (4.5,1.25) -- (8.25, 1.25);
\draw [line width=0.22mm] (4.25,1.0) -- (4.25, 0.25);
\draw [line width=0.22mm] (0.25,2) -- (0.25, 1.25);
\draw [line width=0.22mm] (8.25,2) -- (8.25, 1.25);
\draw [line width=0.22mm] (0.05, 3.25) rectangle (0.45,3.65);
\draw [line width=0.22mm] (2, 3.25) rectangle (2.5,3.75);
\draw [line width=0.22mm] (4.05, 3.25) rectangle (4.45,3.65);
\draw [line width=0.22mm] (6, 3.25) rectangle (6.5,3.75);
\draw [line width=0.22mm] (8.05, 3.25) rectangle (8.45,3.65);
\draw [line width=0.22mm] (4, 0.25) rectangle (4.5,-0.25);
\node at (0.25, 2.25){$+$};
\node at (2.25, 2.20){$=$};
\node at (4.25, 2.25){$+$};
\node at (6.25, 2.20){$=$};
\node at (8.25, 2.25){$+$};
\node at (4.25, 1.20){$=$};
 \draw (0.78,3.43) node{$\tilde \phi_1$};
  \draw (4.78,3.43) node{$\tilde \phi_2$};
 \draw (8.78,3.43) node{$\tilde \phi_3$};
 \draw (2.78,3.52) node{$\tilde \psi_1$};
 \draw (6.78,3.52) node{$\tilde \psi_2$};
  \draw (3.75,0.0) node{$\tilde \psi_3$};
  \draw (0.0,2.86) node{$\tilde X_1$};
  \draw (3.98,2.86) node{$\tilde X_2$};
 \draw (7.98,2.86) node{$\tilde X_3$};
 \draw (2.0,2.86) node{$\tilde Y_1$};
  \draw (6.0,2.86) node{$\tilde Y_2$};
    \draw (4.5,0.61) node{$\tilde Y_3$};
\draw[fill=white] (2.584,2.25) circle (0.82mm);
\draw[fill=white] (6.584,2.25) circle (0.82mm);
\draw[fill=white] (3.911,1.25) circle (0.82mm);
  \end{tikzpicture}
  \caption{\label{fig:2DGridModDual}
The dual NFG of Fig.~\ref{fig:2DGridMod},
where the big unlabeled boxes represent~(\ref{eqn:PottsDJ}), and 
the small unlabeled boxes represent~(\ref{eqn:PottsDV}). Boxes labeled ``$=$'' are instances of equality 
indicator factors as in (\ref{eqn:Definition1}),
the boxes labeled ``$+$'' are instances of zero-sum indicator 
factors as in (\ref{eqn:Definition2}), and the symbol ``$\circ$''
models a sign inversion.}
\end{figure}

\section{The Dual NFG}
\label{sec:Dual}

We can obtain the dual of a primal NFG representing (\ref{eqn:ProbP}) by employing the following steps.
\begin{itemize}
\item Replace each variable, say $X_k$, by its corresponding dual variable $\tilde X_k$.
\item Replace each factor $\psi_e(\cdot)$ by its 1D DFT $\tilde\psi_e(\cdot)$. 
\item Replace each factor $\phi_v(\cdot)$ by its 1D DFT $\tilde\phi_v(\cdot)$. 
\item Replace equality indicator factors 
by zero-sum indicator factors, and vice-versa.
\end{itemize}

The above dualizations procedure retains the topology of the primal NFG. 
In general, dual variables take values in the dual of $\calA$ (i.e., its character group). However, the dual group can be 
taken as $\calA = \ZZ/q\ZZ$ in our framework. We will use the tilde symbol to denote variables in the dual domain.

Here $\tilde \psi_e(\cdot)$ the 1D DFT of $\psi_e(\cdot)$ can be computed as
\begin{equation}
\label{eqn:FTformula}
\tilde \psi_e(\tilde y_e) = \sum_{y \in \calA} \psi_e(y_e)\omega^{y_e\tilde y_e}_{|\calA|},
\end{equation} 
where $\omega_{|\calA|} = \mathrm{e}^{-2\pi \mathrm{i}/|\calA|}$ and $\mathrm{i}=\sqrt{-1}$. See~\citep{Brace:99} for more details on 
the DFT and~\citep{Terras1999} for more
details on the Fourier transform
on finite abelian groups

After applying the dualization procedure and after a little rearranging, we obtain the dual NFG of the 1D Potts model in an 
external field with periodic boundaries (i.e., the dual of~\Fig{fig:2DGridMod}), which is illustrated in~\Fig{fig:2DGridModDual}. 

The big unlabeled boxes in~\Fig{fig:2DGridModDual} are the 1D DFT of~(\ref{eqn:PottsPot1}), which can be computed as
\begin{align}
\tilde \psi_e(\tilde y_e) & = \sum_{y_e \in \calA} \mathrm{e}^{\beta J_e\delta(y_e)}\omega^{y_e\tilde y_e}_{|\calA|} \\
& = \sum_{y_e \in \calA} \big(1+ (\mathrm{e}^{\beta J_e}-1)\delta(y_e)\big)\omega^{y_e\tilde y_e}_{|\calA|},
\end{align} 
which, after a little algebra, gives
\begin{equation} 
\label{eqn:PottsDJ}
\tilde \psi_{e}(\tilde y_e) = \left\{ \begin{array}{ll}
      \mathrm{e}^{\beta J_e} -1 + q, & \text{if $\tilde y_e = 0$} \\
      \mathrm{e}^{\beta J_e} - 1, & \text{otherwise.}
  \end{array} \right.
\end{equation}
Similarly, the small unlabeled boxes are the 1D DFT of~(\ref{eqn:PottsPot2}) given by
\begin{equation} 
\label{eqn:PottsDV}
\tilde \phi_{v}(\tilde x_v) = \left\{ \begin{array}{ll}
      \mathrm{e}^{\beta H_v} - 1 + q, & \text{if $\tilde x_v = 0$} \\
      \mathrm{e}^{\beta H_v} - 1, & \text{otherwise.}
  \end{array} \right.
\end{equation}

In the dual NFG, it holds that $\tilde\x(\tilde\y) = \pmb{M}^\mathsf{T}\tilde \y$, where $\pmb{M}^\mathsf{T}$ denotes the transpose of $\pmb{M}$. In 
other
words, we can freely choose $\tilde \Y$ and therefrom compute $\tilde \X$.  
E.g., if we take $\G$ to be a $d$-dimensional model with pairwise interactions 
and assume periodic boundary conditions, each component $\tilde X_v$ of $\tilde \X$  can be computed 
from the $2d$ components of $\tilde \Y$ that are incident to the zero-sum 
indicator factor attached to $\tilde X_v$.

If a ferromagnetic Potts model is under the influence of a nonnegative external field, both
factors (\ref{eqn:PottsDJ}) and (\ref{eqn:PottsDV}) will be nonnegative. 
The global PMF in the dual NFG can thus be defined as
\begin{equation}
\label{eqn:Pd}
\mathrm{\pi}_\text{d}(\tilde \y) = \frac{1}{Z_\text{d}}\prod_{e \in \EE} \tilde \psi_e(\tilde y_e) \prod_{v \in \VV} \tilde \phi_v(\tilde x_v),
\end{equation} 
where
\begin{equation}
\label{eqn:Zd}
Z_\text{d} = \sum_{\tilde \y \in \calA^{|\EE|}} \prod_{e \in \EE} \tilde \psi_e(\tilde y_e) \prod_{v \in \VV} \tilde \phi_v(\tilde x_v)
\end{equation} 
 is the partition function in the dual domain.

In the models that we study in this paper, the NFG duality theorem states that the partition function of the dual NFG $Z_\text{d}$ is
equal to the partition function of the primal NFG $Z_\text{p}$, up to some know scale factor. Indeed
\begin{equation}
\label{eqn:scaleF}
Z_{\text{d}} = \alpha(\G)Z_\text{p}
\end{equation} 
For more details, see~\citep{AY:2011}. The scale factor $\alpha(\G)$ depends on the topology of $\G$, and is given by
\begin{equation}
\label{eqn:scaleFG}
\alpha(\G) = |\calA|^{|\EE| - |\VV| + 1}
\end{equation} 
The term $|\EE| - |\VV| + 1$ that appears in the exponent in (\ref{eqn:scaleFG}) is equal to the first Betti number 
(i.e., the cyclomatic number) of $\G$. See~\citep[Appendix]{Mo:Allerton17} and \citep{Forney:18} for more details.

The NFG duality theorem as expressed in (\ref{eqn:scaleF}) is closely related to the Holant Theorem~\citep{valiant2008holographic}. 
We refer the readers who are interested in this connection to~\citep{cai2008holographic}, \citep[Section 3.2]{AY:2011}, and~\citep{cai2017complexity}.



\subsection{Dual NFG of the Ising Model and the High-Temperature Series Expansion}

In the dual NFG of the Ising model 
\begin{equation} 
\label{eqn:IsingDJ}
\tilde \psi_{e}(\tilde y_e) = \left\{ \begin{array}{ll}
      2\cosh(\beta J_e), & \text{if $\tilde y_e = 0$} \\
      2\sinh(\beta J_e), & \text{if $\tilde y_e = 1,$}
  \end{array} \right.
\end{equation}
which is the 1D DFT of (\ref{eqn:IsingPot1}). Similarly 
\begin{equation} 
\label{eqn:IsingDH}
\tilde \phi_{v}(\tilde x_v) = \left\{ \begin{array}{ll}
      2\cosh(\beta H_v), & \text{if $\tilde x_v = 0$} \\
      2\sinh(\beta H_v), & \text{if $\tilde x_v = 1$}
  \end{array} \right.
\end{equation}
is the 1D DFT of (\ref{eqn:IsingPot2}). 

As a side remark, note that small and large values of the coupling parameters correspond to
the high-temperature and low-temperature regimes, respectively.

Let us consider a ferromagnetic Ising model in a constant and positive external field $H$. The valid configurations in the dual NFG of this model give 
rise to the following expansion of the partition function 
\begin{equation} \label{eqn:HTExpandExt}
Z_{\text{p}} \propto \sum_{\calU \subseteq \EE}\tanh(H)^{|\text{odd$(\calU)$}|}\prod_{(k,\ell) \in \,\calU} \tanh(J_e),
\end{equation}
where  $\calU \subseteq \EE$ and $\text{odd$(\calU)$}$ denotes the set of all odd-degree 
vertices in the subgraph of $\EE$
induced by $\calU$. For more details, see~\citep[Section VIII]{MoGo:2018}.

In statistical physics, the sum in~(\ref{eqn:HTExpandExt}) is generally known as the high-temperature series 
expansion of the partition function~\citep{Newell:53} and~\cite[Chapter 10]{Nishimori:15}.

In~\citep{JS:93}, the authors have proposed a fully polynomial-time randomized approximation scheme called the 
subgraphs-world process to estimate the partition function of ferromagnetic Ising models with arbitrary topology and in a positive external field.
The subgraphs-world process is based on the high-temperature series 
expansion of the partition function. Indeed, 
the configurations in the subgraphs-world process coincide with the configurations in~(\ref{eqn:HTExpandExt}).

The mixing time of the subgraphs-world process is polynomial in the size of the model 
at all temperatures.\footnote{This is a remarkable result, as typical Monte Carlo methods (e.g., the Gibbs sampling algorithm for the 2D Ising model) are rapidly mixing only
in the high-temperature regime \citep{martinelli1994approach}.}
It is therefore possible to employ the subgraphs-world process to generate configurations in the dual
NFG of the ferromagnetic Ising model. Since the process is rapidly
mixing, it converges in polynomial time.

%
%
%

\section{Marginal Probabilities and Pertinent Mappings}
\label{sec:Marginals}

In this section, we derive local mappings that relate marginal probabilities at the edges of a primal NFG to the
marginal probabilities at the corresponding edges in the dual NFG. Indeed, we will prove 
that 
$(\pi_{\text{p}, e}(a)/\psi_e(a), a\in \calA)$ in the primal domain 
and $(\pi_{\text{d}, e}(a)/\tilde \psi_e(a), a \in \calA)$ in the dual domain are 
Fourier pairs.

In the primal NFG
the marginal probability over $e \in \EE$ can be computed as
\begin{equation}
\label{eqn:MargProbP1}
\pi_{\text{p}, e}(a)  = \frac{Z_{\text{p}, e}(a)}{Z_\text{p}}, \quad a \in \calA,
\end{equation}
where
\begin{equation}
Z_{\text{p}, e}(a) = \sum_{\x\colon y_e(\x) = a} \, \prod_{e \in \EE} \psi_e(y_e(\x)) \prod_{v \in \VV}\phi_v(x_v)
\label{eqn:MargProbUD}
\end{equation}
Equivalently,
\begin{equation}
Z_{\text{p}, e}(a) = S_e(a)\psi_e(a)
\label{eqn:MargProbUD2}
\end{equation}
with
\begin{equation}
S_e(a) = \sum_{\x\colon y_e(\x) = a}\,\prod_{\e \in \EE \setminus \{e\}} \psi_\e(y_\e(\x)) \prod_{v \in \VV}\phi_v(x_v)
\label{eqn:MargProbUD3}
\end{equation}

The partition function $Z_{\text{p}}$ can therefore be written as the dot product of the 
vectors $S_e(\cdot)$ and $\psi_e(\cdot)$ as
\begin{align}
\label{eqn:Dotproduct}
Z_{\text{p}} &= \langle  S_e, \psi_e\rangle \\
                      &= \sum_{a \in \calA} S_e(a)\psi_e(a),
\end{align}
which indicates that (\ref{eqn:MargProbP1}) is a valid PMF over $\calA$.

In Coding Theory terminology, $\left(\psi_e(a), a \in \calA \right)$ is called the \emph{intrinsic} message 
vector and $\left(S_e(a), a \in \calA \right)$ is called the \emph{extrinsic} message vector at edge $e \in \EE$. 
The fact that at any edge $e \in\EE$ the partition function is the dot product
of the intrinsic and extrinsic message vectors is called the \emph{sum-product rule}. 
This observation holds for any
edge that is by itself a cut-set of $\G$.
For more details, see~\cite{Forney:01}.

\begin{figure}[t]
  \centering
  \begin{tikzpicture}[scale=1.0]

\draw [line width=0.22mm] (0, 2) rectangle (0.5,2.5);
\draw [line width=0.22mm] (2, 2) rectangle (2.5,2.5);
\draw [line width=0.22mm] (4, 2) rectangle (4.5,2.5);
\draw [line width=0.22mm] (2, 3.25) rectangle (2.5,3.75);
\draw [line width=0.22mm] (0.5,2.25) -- (2, 2.25);
\draw [line width=0.22mm] (2.5,2.25) -- (4, 2.25);
\draw [line width=0.22mm] (2.25,2.5) -- (2.25, 3.25);
\node at (0.25, 2.20){$=$};
\node at (2.25, 2.25){$+$};
\node at (4.25, 2.20){$=$};
\draw [line width=0.22mm] (7.5,2) rectangle (8,2.5);
\draw [line width=0.22mm] (9.5, 2) rectangle (10,2.5);
\draw [line width=0.22mm] (11.5, 2) rectangle (12,2.5);
\draw [line width=0.22mm] (9.5, 3.25) rectangle (10,3.75);
\draw [line width=0.22mm] (9.5, 3.25) rectangle (10,3.75);
\node at (7.75, 2.25){$+$};
\node at (9.75, 2.20){$=$};
\node at (11.75, 2.25){$+$};
\draw [line width=0.22mm] (8,2.25) -- (9.5, 2.25);
\draw [line width=0.22mm] (10,2.25) -- (11.5, 2.25);
\draw [line width=0.22mm] (9.75,2.5) -- (9.75, 3.25);
 \draw (2.78,3.52) node{$\xi_e$};
 \draw (10.28,3.52) node{$\tilde \xi_e$};
 \draw (2.0,2.86) node{$Y_e$};
 \draw (9.5,2.86) node{$\tilde Y_e$};
\draw[fill=white] (2.584,2.25) circle (0.82mm);
\draw[fill=white] (10.084,2.25) circle (0.82mm);
\draw[myarrow] (5.44,2.3) -- (6.54,2.3);
\node at (5.9, 2.65){\emph{dual}};
  \end{tikzpicture}
  \vspace{1.0ex}
  \caption{\label{fig:Marg2}
An edge $e \in \EE$ in the intermediate primal NFG (left) and in the intermediate dual  NFG (right).
The unlabeled box (left) represents $\xi_e$ given by~(\ref{eqn:Int1}) and the unlabeled box (right) represents 
$\tilde \xi_e$, the 1D DFT of $\xi_e$.}
\end{figure}
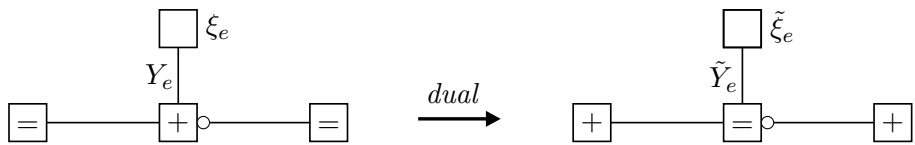



Alternatively, $S_e(a)$ can be viewed as the partition function of 
an intermediate primal NFG in which all factors are equal to their corresponding factors in the primal NFG, except for 
$\psi_e(y_e)$, which is replaced by
\begin{equation}
\label{eqn:Int1}
\xi_e(y_e; a) = \delta(y_e - a),
\end{equation}
as shown in \Fig{fig:Marg2} (left).

The dual of the intermediate primal 
NFG is shown in~\Fig{fig:Marg2} (right), in which, equality indicator factors are replaced by zero-sum indicator factors, 
and $\xi_e(\cdot)$ is replaced by $\tilde \xi_e(\cdot)$, the 1D DFT of $\xi_e(\cdot)$. 
We denote the partition function of the dual intermediate NFG by $Z_{\text{d}}^{\text{I}}(a)$.
From the NFG duality theorem~(\ref{eqn:scaleF}), we obtain
\begin{equation}
\label{eqn:ZDI}
Z_{\text{d}}^{\text{I}}(a) = \alpha(\G)\hspace{0.05mm}S_e(a)
\end{equation}

Similarly, the edge marginal probability at $e \in \EE$ of the dual NFG is given by
\begin{equation}
\label{eqn:MargProbD}
\pi_{\text{d}, e}(a)  = \frac{Z_{\text{d}, e}(a)}{Z_{\text{d}}}, \quad a \in \calA,
\end{equation}
where
\begin{equation}
Z_{\text{d}, e}(a) = \sum_{\tilde \y\colon \tilde y_e = a} \, 
\prod_{e \in \EE} \tilde \psi_e(\tilde y_e) \prod_{v \in \VV}\tilde \phi_v(\tilde x_v(\y))
\label{eqn:MargProbUD}
\end{equation}
The partition function of the dual NFG can be expressed by the following dot product
\begin{align}
\label{eqn:DotproductDual}
Z_{\text{d}} &= \langle  T_e, \tilde \psi_e\rangle \\
                      &= \sum_{a \in \calA} T_e(a)\tilde \psi_e(a) \label{eqn:DotproductDualt2}
\end{align}
with
\begin{equation}
\label{eqn:Dual3}
T_e(a) = \sum_{\tilde \y\colon \tilde y_e = a}\,
             \prod_{\e \in \EE \setminus \{e\}} \tilde \psi_\e(\tilde y_\e) \prod_{v \in \VV}\tilde \phi_v(\tilde x_v(\y))
\end{equation} 

\begin{proposition}
\label{prop:EdgeDFT}
The vectors $(\pi_{\mathrm{p}, e}(a)/\psi_e(a), a\in \calA)$ and 
$(\pi_{\mathrm{d}, e}(a)/\tilde \psi_e(a), a \in \calA)$ are DFT pairs.
\end{proposition}

\begin{proof}
The partition function of the dual intermediate NFG is equal to the dot product of 
the vectors $( T_e(a), a \in \calA)$ and $(\tilde \xi_e(\tilde y_e; a), a \in \calA)$. 
From (\ref{eqn:ZDI}), we obtain
\begin{align}
\alpha(\G)\hspace{0.05mm}S_e(a) & = \langle  T_e, \tilde \xi_e\rangle \\
              & = \sum_{a'\in \calA} T_e(a')\Big(\sum_{y \in \calA}\delta(y-a)\omega^{ya'}_{|\calA|}\Big)\\
              & = \sum_{a'\in \calA} T_e(a')\hspace{0.07mm}\omega^{aa'}_{|\calA|} \label{eqn:DotproductInterm}
\end{align} 
for $a \in \calA$. We conclude that the vectors $(S_e(a), a \in \calA)$ and $(T_e(a), a \in \calA)$ 
are Fourier pairs, up to scale factor $\alpha(\G)$.



On the other hand (\ref{eqn:MargProbP1}) and (\ref{eqn:MargProbUD2}) give
\begin{equation}
\label{eqn:MargProbInterim2}
S_e(a) = Z_{\text{p}}\hspace{0.07mm} \frac{\pi_{\text{p}, e}(a)}{\psi_e(a)}, \quad a \in \calA,
\end{equation}
and combining (\ref{eqn:MargProbD}) and (\ref{eqn:DotproductDualt2}) yields
\begin{align}
T_e(a) & = Z_{\text{d}}\hspace{0.07mm}\myfracc{\pi_{\text{d}, e}(a)}{\tilde \psi_e(a)}\\
& = \alpha(\G)\hspace{0.07mm}Z_{\text{p}}\myfracc{\pi_{\text{d}, e}(a)}{\tilde \psi_e(a)},
  \quad a \in \calA, \label{eqn:MargProbInterim3}
\end{align}
where (\ref{eqn:MargProbInterim3}) follows from the NFG duality theorem in (\ref{eqn:scaleF}).

Substituting $(S_e(a), a \in \calA)$ in (\ref{eqn:MargProbInterim2}) and 
$(T_e(a), a \in \calA)$ in (\ref{eqn:MargProbInterim3}) into~(\ref{eqn:DotproductInterm}) yields
\begin{equation}
\label{eqn:Proposition1Proof}
\frac{\pi_{\text{p}, e}(a)}{\psi_e(a)} = \sum_{a'\in \calA} \myfracc{\pi_{\text{d}, e}(a')}{\tilde \psi_e(a')}\hspace{0.07mm}\omega^{aa'}_{|\calA|},
\end{equation}
which completes the proof.
\end{proof}

We state without proof that
\begin{proposition}
\label{prop:VertexDFT}
The vectors $(\pi_{\mathrm{p}, v}(a)/\phi_v(a), a \in \calA)$ 
and $(\pi_{\mathrm{d}, v}(a)/\tilde \phi_v(a), a \in \calA)$ are DFT pairs. 
\end{proposition}

By virtue of Propositions \ref{prop:EdgeDFT} and \ref{prop:VertexDFT}, it is possible to estimate
the edge marginal probabilities in one domain, and then transform them to the other domain all together.
In particular, for the ferromagnetic Ising model in a positive external field we can employ the
subgraphs-world process (as a rapidly mixing Markov chain) to generate configuration in the
dual NFG of the model.

We finish this section with the following important remarks.

\begin{remark} 
The proposed mappings are fully local (i.e., edge-dependent or vertex-dependent). Moreover, the mappings do not depend on
the size and on the topology of~$\G$. 
Indeed, all the relevant information regarding the rest of the graph is already incorporated in the 
estimated edge marginal densities.
\end{remark}

\begin{remark} 
Transforming marginals from one domain to the other requires computing a DFT with
computational complexity $\mathcal{O}(|\calA|^2)$, which can be reduced to $\mathcal{O}(|\calA|\log(|\calA|))$ via
the fast Fourier transform (FFT). 
However, when there is symmetry in the factors, as in the case of the Ising and the Potts models, the complexity can be 
further reduced to $\mathcal{O}(|\calA|)$.
\end{remark}

\begin{remark} 
In binary models, the factors in the dual NFG can in general take negative values, 
and in nonbinary models, the factors can even be complex-valued. 
In such cases a 
valid PMF can no longer be defined in the 
dual domain.The mapping nevertheless remains valid; but for 
marginal functions, instead of marginal densities, 
of a global function with a factorization given by~(\ref{eqn:Pd}). 
\end{remark}

\section{Details of the Mapping for Binary Models}
\label{sec:IsingModelMapping}


In binary models (i.e., when $\calA = \ZZ/2\ZZ$), Proposition \ref{prop:EdgeDFT} provides the following mapping
\begin{equation}
\label{eqn:MapG}
\begin{pmatrix} \, \pi_{\text{p}, e}(0)/\psi_e(0)\,  \\[8pt] \, \pi_{\text{p}, e}(1)/\psi_e(1)\, \end{pmatrix} = 
\begingroup
\renewcommand*{\arraystretch}{1.0}
\begin{pmatrix}      
\, 1 & \phantom{-}1\, \\[6pt] \,1 & - 1\,
\end{pmatrix}
\endgroup
\begin{pmatrix} \, \pi_{\text{d}, e}(0)/\tilde \psi_e(0) \, \\[8pt] \, \pi_{\text{d}, e}(1)/\tilde \psi_e(1)\, \end{pmatrix}
\end{equation}
in matrix-vector format via the two-point DFT matrix,
where 
\begin{equation} 
\label{eqn:EdgeD1}
\tilde \psi_{e}(0) = \psi_{e}(0) + \psi_{e}(1) 
\end{equation} 
and
\begin{equation}
\label{eqn:EdgeD2}      
\tilde \psi_{e}(1) = \psi_{e}(0) - \psi_{e}(1)
\end{equation}

For a general Ising model, substituting factors~(\ref{eqn:IsingPot1}) and~(\ref{eqn:IsingDJ}) in (\ref{eqn:MapG}) yields 
\begin{equation}
\label{eqn:MapDP}
\begin{pmatrix} \, \pi_{\text{p},e}(0)\,  \\[6pt] \,\pi_{\text{p},e}(1)\, \end{pmatrix} = 
\begingroup
\renewcommand*{\arraystretch}{1.9}
\begin{pmatrix}      
\,\dfrac{\textrm{e}^{\beta J_e}}{2\cosh(\beta J_e)} & \hphantom{-}\dfrac{\textrm{e}^{\beta J_e}}{2\sinh(\beta J_e)}\, \\[8pt] 
\,\dfrac{\textrm{e}^{-\beta J_e}}{2\cosh(\beta J_e)} & - {\dfrac{\textrm{e}^{-\beta J_e}}{2\sinh(\beta J_e)}}\,
\end{pmatrix}
\endgroup
\begin{pmatrix} \, \pi_{\text{d},e}(0)\, \\[6pt] \, \pi_{\text{d},e}(1)\, \end{pmatrix} 
\end{equation}
for $\beta J_e \ne 0$, which implies
\begin{equation}
\label{eqn:MapIsingInverse}
\begin{pmatrix} \, \pi_{\text{d},e}(0)\,  \\[6pt] \, \pi_{\text{d},e}(1)\, \end{pmatrix} = 
\begingroup
\renewcommand*{\arraystretch}{1.9}
\begin{pmatrix}      
\,\dfrac{\cosh(\beta J_e)}{\textrm{e}^{\beta J_e}} & \phantom{-}\dfrac{\cosh(\beta J_e)}{\textrm{e}^{-\beta J_e}}\, \\[8pt] \,\dfrac{\sinh(\beta J_e)}{\textrm{e}^{\beta J_e}} & - \dfrac{\sinh(\beta J_e)}{\textrm{e}^{-\beta J_e}}\,
\end{pmatrix}
\endgroup
\begin{pmatrix} \, \pi_{\text{p},e}(0)\,  \\[6pt] \,\pi_{\text{p},e}(1)\, \end{pmatrix}
\end{equation}

As a sanity check, we consider a ferromagnetic and homogeneous Ising model in the low-temperature and high-temperature 
limits. 
In the primal NFG and in the low-temperature limit (i.e., as $\beta J_e \to {+}\infty$), 
the edge-marginal probability $\pi_{\text{p},e}(0) = 1- \pi_{\text{p},e}(1) = 1$, which is in agreement with (\ref{eqn:MapDP}).
In the high-temperature limit (i.e., as $\beta J_e \to 0$), the marginal probability over edge $e$ of the dual NFG is 
$\pi_{\text{d},e}(0) = 1- \pi_{\text{d},e}(1) = 1$. This is in agreement with (\ref{eqn:MapIsingInverse}).
(Recall that $\beta$ is proportional to the inverse temperature.)

\begin{proposition}
\label{prop:2DIsingBounds}
In an arbitrary ferromagnetic Ising model in a nonnegative external field, it holds that
\end{proposition}
\begin{equation}
\label{eqn:BoundonPe0}
\pi_{\mathrm{p},e}(0) \ge \dfrac{1}{1+ \textrm{e}^{-2\beta J_e}}
\end{equation}
\emph{and}
\begin{equation}
\label{eqn:BoundonPd0}
\pi_{\mathrm{d},e}(0) \ge \dfrac{1+\textrm{e}^{-2\beta J_e}}{2}
\end{equation}

\begin{proof} In the dual domain, the global PMF of a ferromagnetic Ising model 
in a nonnegative field $\pi_{\text{d},e}(\cdot)$ is given by~(\ref{eqn:Pd}). From~(\ref{eqn:MapDP})
\begin{align}
\myfrac{\pi_{\text{p},e}(0)}{\textrm{e}^{\beta J_e}}  
& = \frac{\pi_{\text{d},e}(0)}{2\cosh(\beta J_e)} + \frac{\pi_{\text{d},e}(1)}{2\sinh(\beta J_e)}\\
& = \frac{1}{2\sinh(\beta J_e)} - \frac{\textrm{e}^{-\beta J_e}}{\sinh(2\beta J_e)}\pi_{\text{d},e}(0) \label{eqn:BoundonPeDerive}
\end{align}
Therefore $\pi_{\text{p},e}(0)$ achieves its minimum when $ \pi_{\text{d},e}(0)= 1$. 
After substituting $\pi_{\text{d},e}(0) = 1$ in~(\ref{eqn:BoundonPeDerive}), and after a little algebra, we obtain
\begin{align}
\pi_{\text{p},e}(0) & \ge \frac{\textrm{e}^{\beta J_e}}{2\cosh(\beta J_e)} \\
                             & = \myfrac{1}{1+\textrm{e}^{-2\beta J_e}} \label{eqn:BoundonProofs}
\end{align}

The proof of~(\ref{eqn:BoundonPd0}) follows along the same lines by starting from (\ref{eqn:MapIsingInverse}), and is omitted.
\end{proof}

From (\ref{eqn:BoundonPe0}) and (\ref{eqn:BoundonPd0}), we conclude that
in an arbitrary ferromagnetic Ising model in a nonnegative external field
\begin{equation}
\label{eqn:Uncertainty}
\pi_{\text{p},e}(0)\pi_{\text{d},e}(0) \ge \frac{1}{2},
\end{equation}
which is in the form of an uncertainty principle.

Fig.~\ref{fig:PBound} shows the lower bounds in (\ref{eqn:BoundonPe0}) and in (\ref{eqn:BoundonPd0}) as a 
function of $\beta J_e$. 
Interestingly, the lower bounds 
intersect at $\beta J_\text{c}=  \ln(1+\sqrt{2})/2$, which coincides with the critical coupling 
(i.e., the phase transition) of  the 2D homogeneous Ising model in zero field and in the thermodynamic limit (i.e., as $|\VV| \to \infty$).
 For more details, see~\citep{kramers1941statistics, Onsager:44}.

We stress out that the lower bonds in Proposition~\ref{prop:2DIsingBounds} are valid  for arbitrary ferromagnetic 
Ising models in a nonnegative external field. The bounds do not depend on the size or on 
the topology of $\G$. In Appendix B, we will prove that the edge marginal probability of the 1D Ising model with free boundary 
conditions attains the lower bound in (\ref{eqn:BoundonPe0}).


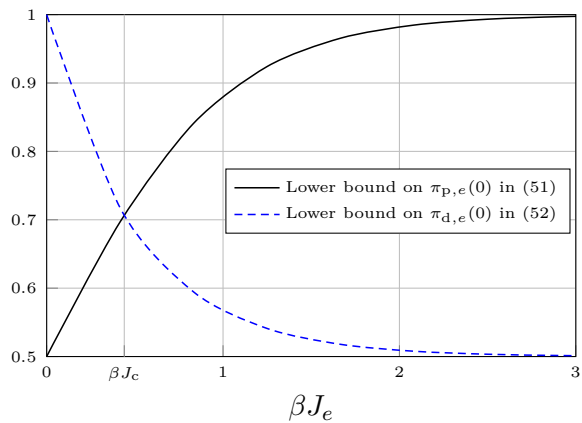
\begin{figure}[t!!]
\centering
\begin{tikzpicture}
\begin{axis}[
			legend style={at = {(0.99,0.55)} ,font=\tiny},		
			height = 37.0ex,
			width = 52ex,
			grid = major,
			tick pos=left, 
			xlabel shift = -2 pt,
			restrict y to domain = 0.5:1, 
			xminorticks = false,	
		    yminorticks = false,	
		    y tick label style={
        /pgf/number format/.cd,
            fixed,
        /tikz/.cd
    		}, 				
			ytick={0.5, 0.6, 0.7, 0.8, 0.9, 1.0},
			xtick={0.0, 1.0, 2.0, 3.0},
			extra x ticks=0.44,
			extra x tick labels={$\beta J_\mathrm{c}$},
		xlabel= $\beta J_e$ ={font=\normalsize},
			xmin = 0.0,
			xmax = 3.0,
			ymin = 0.5,
			ymax = 1.0,
			yticklabel style = {font=\tiny,yshift=0.0ex},
            xticklabel style = {font=\tiny,xshift=0.0ex}			
			]

\addplot[black, line width = 0.58 pt, smooth]    {e^(2*x)/(1+e^(2*x))};
\addplot[blue, line width = 0.58 pt, smooth, densely dashed]    {(1+e^(2*x))/(2*e^(2*x))};

 		 \legend{Lower bound on $\pi_{\text{p},e}(0)$ in~(\ref{eqn:BoundonPe0}), Lower bound on $\pi_{\text{d},e}(0)$ in~(\ref{eqn:BoundonPd0})};	  	
\end{axis}
\end{tikzpicture}
\caption{\label{fig:PBound}%
For a ferromagnetic Ising model in a nonnegative external field,
the solid black plot and the dashed blue plot show the lower bound on $\pi_{\text{p},e}(0)$ given by (\ref{eqn:BoundonPe0}), and 
the lower bound on $\pi_{\text{d},e}(0)$ given by (\ref{eqn:BoundonPd0}), as a function of $\beta J_e$, respectively. 
The lower bounds intersect at  the criticality of the 
2D homogeneous Ising model in the absence of an external field, denoted by $\beta J_\mathrm{c}$.}
\end{figure}

\subsection{The Fixed Points}

Let us denote the fixed points of the mapping in (\ref{eqn:MapG}) by $(\pi^{*}_{e}(a), a\in \calA)$. A routine calculation 
gives 
\begin{equation}
\label{eqn:fixedGeneralBinary}
\begin{pmatrix} \, \pi^{*}_{e}(0)  & \pi^{*}_{e}(1)\, \end{pmatrix} = \\
\begin{pmatrix} \, \dfrac{\psi_{e}(0)\tilde \psi_{e}(0)}{\psi_{e}(0)\tilde \psi_{e}(0)+\vphantom{\tilde \psi^{2^2}}\psi_{e}(1)\tilde \psi_{e}(1)} \, & \,
\dfrac{\psi_{e}(1)\tilde \psi_{e}(1)}{\psi_{e}(0)\tilde \psi_{e}(0)+\vphantom{\tilde \psi^{2^2}}\psi_{e}(1)\tilde \psi_{e}(1)}\, \end{pmatrix}
\end{equation}
For a homogeneous and ferromagnetic Ising model we thus obtain
\begin{equation}
\label{eqn:fixed}
\begin{pmatrix} \, \pi^{*}(0)  & \pi^{*}(1)\, \end{pmatrix} = \\
\begin{pmatrix} \, \dfrac{\textrm{e}^{\beta J}\cosh (\beta J)}{1+\sinh(2\beta J)}\, & \,\dfrac{\textrm{e}^{-\beta J}\sinh(\beta J)}{1+\sinh(2\beta J)}\, \end{pmatrix}
\end{equation}
Fig.~\ref{fig:L1} shows the fixed points $\pi^{*}(\cdot)$ as a function of $\beta J$. 

\begin{proposition}
\label{prop:2DHom}
The minimum of $\pi^{*}(0)$ and the maximum of $\pi^{*}(1)$ are attained at the criticality of the 
2D homogeneous Ising model in the absence of an external magnetic field. 
\end{proposition}

\begin{proof} 
From (\ref{eqn:HamiltonianIsing}), the Hamiltonian of the model is 
\begin{align}
\label{eqn:HamP}
\mathcal{H}(\y)  & = -\sum_{e \in \EE}J\big(2\delta(y_e) -1\big)\\
						  & = -J\sum_{e \in \EE}(1-2y_e)
\end{align}
Consequently, the average energy can be computed as 
\begin{align}
\label{eqn:aveE}
\overline{\mathcal{H}}(\y) & = \sum_{\y \in \calA}\pi_\text{p}(\y)\mathcal{H}(\y)\\ 
	& = -J|\EE|(1-2\E[Y_e])\\
	& = -J|\EE|(1-2\pi_{\text{p},e}(1))
\end{align}
In a 2D Ising model with periodic boundaries $|\EE| = 2|\VV|$, therefore the average energy
per site is equal to 
\begin{equation}
\label{eqn:AveEper}
\overline{\mathcal{H}}(\y)/|\VV|  = -2J(1-2\pi_{\text{p},e}(1))
\end{equation}

From Onsager's closed-form solution, in the thermodynamic limit, the logarithm of the partition function (i.e., the free 
energy) per site of the 2D homogeneous Ising model in zero field is
given by
\begin{equation}
\label{eqn:ZOnsg}
\lim_{|\VV| \to \infty} \frac{\ln Z(\beta J)}{|\VV|} = \frac{1}{2}\ln(2\cosh^{2}2\beta J)  +
\frac{1}{\pi}\int_{0}^{\frac{\pi}{2}}\ln\Big(1+ (1-\kappa^2\sin^2\theta)^{1/2}\,\Big)\,d\theta
\end{equation}
and the average (internal) energy per site $U(\beta J)$ is
\begin{align}
\label{eqn:IntEPerSite}
U(\beta J) & = {-}\lim_{|\VV| \to \infty}\frac{1}{|\VV|}\cdot\frac{\partial \ln Z(\beta J)}{\partial \beta} \\
& = -J\coth(2\beta J)\Big(1-\frac{1}{2\pi}(1-\kappa\sinh2\beta J) \int_{0}^{\frac{\pi}{2}}
 (1-\kappa^2\sin^2\theta)^{-1/2} d\theta\Big)
\end{align}
with
\begin{equation}
\label{eqn:Cons}
\kappa(\beta J) = \myfrac{2\sinh2\beta J}{\cosh^{2}2\beta J}
\end{equation}
For more details, see~\citep{Onsager:44} and \citep[Chapter 7]{Baxter:07}.

At criticality (i.e., at $\beta J_\text{c} =  \ln(1+\sqrt{2})/2$) it is easy to show that $\kappa(\beta J_\text{c}) =1$ and the average energy 
per site $U(\beta J_\text{c})$ is 
\begin{equation}
\label{eqn:IntEPT}
U(\beta J_\text{c}) = -\sqrt{2}J_\text{c}
\end{equation}

From (\ref{eqn:AveEper}) and (\ref{eqn:IntEPT}), we 
conclude that  in the thermodynamic limit
\begin{equation}
\label{eqn:fixedC}
\left .\begin{pmatrix} \, \pi(0)  & \pi(1)\, \end{pmatrix}\right\vert_{\beta J = \beta J_\text{c}}  =
\begin{pmatrix} \, (2+\sqrt{2})/4\, & \, (2-\sqrt{2})/4\, \end{pmatrix}, 
\end{equation}
which coincides with the minimum of $\pi^{*}(0)$ and the maximum
of $\pi^{*}(1)$ in~(\ref{eqn:fixed}).
\end{proof}

The fixed points at $\beta J_\text{c}$
are illustrated by filled circles in Fig.~\ref{fig:L1}.
We emphasize that, in the thermodynamic limit and in the absence of an external field, the edge marginal probabilities in the primal and in the dual of 
the 2D homogeneous Ising model are equal at criticality (i.e., at the phase transition).

As side remark, we point out that the 2D Ising model and its criticality have been studies in the context of generative neural networks and restricted 
Boltzmann machines, see~\citep{melko2018} and \citep{cossu2019machine}.

\begin{figure}[t!!]
\centering
\begin{tikzpicture}
\begin{axis}[
			legend style={at = {(0.98,0.73)} ,font=\tiny},		
			height = 37.0ex,
			width = 52ex,
			grid = major,
			tick pos=left, 
			xlabel shift = -2 pt,
			xminorticks = false,	
		    yminorticks = false,	
		    y tick label style={
        /pgf/number format/.cd,
            fixed,
        /tikz/.cd
    		}, 				
			ytick={0, 0.2, 0.4, 0.6, 0.8, 1.0},
			xtick={0.0, 1.0, 2.0, 3.0},
			extra x ticks=0.44,
			extra x tick labels={$\beta J_\mathrm{c}$},
		xlabel= $\beta J$ ={font=\normalsize},
			xmin = 0.0,
			xmax = 3.0,
			ymin = 0.0,
			ymax = 1.0,
			yticklabel style = {font=\tiny,yshift=0.0ex},
            xticklabel style = {font=\tiny,xshift=0.0ex}			
			]

\pgfplotstableread{./files/Pfixed.txt}\mydataone
\pgfplotstableread{./files/Pfixed2.txt}\mydatatwo

		\addplot [
		line width = 0.58 pt,
		 color = black
		]		
		 table[y = Z] from \mydataone;
		 
 		 \addplot [
 		 densely dashed,
 		 line width = 0.58 pt,
 		 color = blue
 		 ]
 		  table[y = Z] from \mydatatwo;	 
 
		\addplot[mark=*] coordinates {(0.44,0.853554)};
		\addplot[color = blue, mark=*] coordinates {(0.44,0.146446)};


 		 \legend{$\pi^{*}(0)$ in~(\ref{eqn:fixed}), $\pi^{*}(1)$ in~(\ref{eqn:fixed}), Min attained at $\beta J_\mathrm{c}$, 
 		 Max attained at $\beta J_\mathrm{c}$};	  	
\end{axis}
\end{tikzpicture}
\caption{\label{fig:L1}%
The fixed points~(\ref{eqn:fixed}) as a function of 
$\beta J$. 
The filled circles show the fixed points at the criticality of the 2D homogeneous Ising model in zero field as in (\ref{eqn:fixedC}).}
\end{figure}
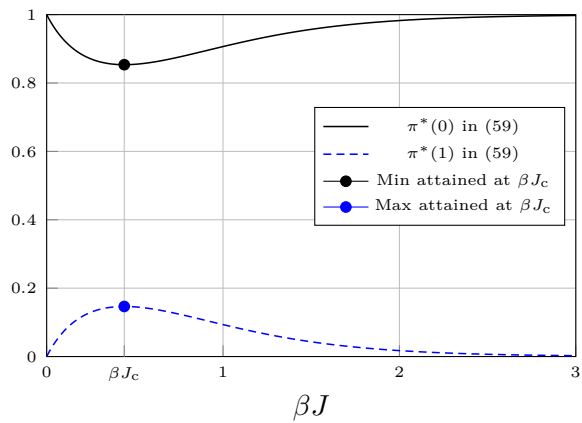

\subsection{The Local Magnetization}
\label{sec:magnet}

In a ferromagnetic Ising model, let 
\begin{equation}
\label{eqn:DeltaE}
\Delta_{\text{p},e} = \pi_{\text{p},e}(0) - \pi_{\text{p},e}(1)
\end{equation}
and 
\begin{equation}
\Delta_{\text{d},e} = \pi_{\text{d},e}(0) - \pi_{\text{d},e}(1)
\end{equation}
denote the local magnetizations at edge $e \in \EE$ of the primal NFG and the dual NFG, respectively.
Clearly, it is possible to compute $(\pi_{\text{p},e}(a), a\in \calA)$ from $\Delta_{\text{p},e}$ and 
$(\pi_{\text{d}, e}(a), a\in \calA)$
from $\Delta_{\text{d},e}$. Furthermore, from the mappings 
in (\ref{eqn:MapDP}) and (\ref{eqn:MapIsingInverse}) we get
\begin{equation}
\label{eqn:PeDeltaD}
\pi_{\text{p},e}(0) = \frac{\textrm{e}^{2\beta J_e} - \Delta_{\text{d},e}}{2\sinh(2\beta J_e)} 
\end{equation}
and
\begin{equation}
\pi_{\text{d},e}(0) = \frac{\coth(\beta J_e) - \Delta_{\text{p},e}}{2\csch(2\beta J_e)},
\end{equation}
where $\csch(\cdot) = 1/\sinh(\cdot)$.

From (\ref{eqn:DeltaE}) and (\ref{eqn:PeDeltaD}) it is easy to verify that
\begin{equation}
\Delta_{\text{p},e} = \frac{\cosh(2\beta J_e) - \Delta_{\text{d},e}}{\sinh(2\beta J_e)},
\end{equation}
which relates the local magnetizations in the primal and in the dual domains.
The local magnetizations $\Delta_{\text{p},e}$ and $\Delta_{\text{d},e}$ are 
both equal to $\sqrt{2}/2$ at $\beta J_\text{c}$ (i.e., at the criticality of the 
2D homogeneous Ising model).

\section{Details of the Mapping for Non-binary Models}
\label{sec:GenP}

In non-binary models (e.g., the Potts and the clock models) the mapping 
between the vecotors $(\pi_{\text{p}, e}(a)/\psi_{e}(a), a \in \calA)$
and $(\pi_{\text{d}, e}(a)/\tilde \psi_{e}(a), a \in \calA)$ is given by
\begin{equation} 
\label{eqn:mappingPottsvector}
\begin{pmatrix} \, \pi_{\text{p}, e}(a)/\psi_{e}(a), a \in \calA \,\end{pmatrix} = 
\pmb{W}_{|\calA| \times |\calA|}
\begin{pmatrix} \, \pi_{\text{d}, e}(a)/\tilde \psi_{e}(a), a \in \calA \,\end{pmatrix},
\end{equation}
where  $(\pi_{\text{p}, e}(a)/\psi_{e}(a), a \in \calA)$
and $(\pi_{\text{d}, e}(a)/\tilde \psi_{e}(a), a \in \calA)$ are column vectors of length $|\calA|$, and
$\pmb{W}_{|\calA| \times |\calA|}$ is the $|\calA|$-point DFT matrix (i.e., the Vandermonde matrix for 
the roots of unity), in which
\begin{equation} 
\label{eqn:Vandermond}
\pmb{W}_{k,\ell} = \omega^{k\ell}_{|\calA|}, \quad k, \ell \in \calA,
\end{equation} 
where $\omega_{|\calA|} = \mathrm{e}^{-2\pi \mathrm{i}/|\calA|}$. 

E.g., in a three-state Potts model the mapping boils down to
\begin{equation}
\label{eqn:MapGP}
\begin{pmatrix} \, \pi_{\text{p}, e}(a)/\psi_e(0)\,  \\[8pt] \, \pi_{\text{p}, e}(a)/\psi_e(1)\,  \\[8pt] \, \pi_{\text{p}, e}(a)/\psi_e(2)\,\end{pmatrix} = 
\begingroup
\renewcommand*{\arraystretch}{1.0}
\begin{pmatrix}      
\, 1 & 1\, & 1\, \\[6pt] \,1 & \omega \, & \omega^2\, \\[6pt]  \,1 & \omega^2 \, & \omega^4\,
\end{pmatrix}
\endgroup
\begin{pmatrix} \, \pi_{\text{d}, e}(0)/\tilde \psi_e(0) \, \\[8pt] \, \pi_{\text{d}, e}(1)/\tilde \psi_e(1)\, \\[8pt] \, \pi_{\text{d}, e}(2)/\tilde \psi_e(2)\,\end{pmatrix} 
\end{equation}
for $\beta J_e \ne 0$, which implies
\begin{equation}
\begin{pmatrix} \, \pi_{\text{d}, e}(0)/\tilde \psi_e(0) \, \\[8pt] \, \pi_{\text{d}, e}(1)/\tilde \psi_e(1)\, \\[8pt] \, \pi_{\text{d}, e}(2)/\tilde \psi_e(2)\,\end{pmatrix} =
\frac{1}{3}
\begingroup
\renewcommand*{\arraystretch}{1.0}
\begin{pmatrix}      
\, 1 & 1\, & 1\, \\[6pt] \,1 & \omega^{-1} \, & \omega^{-2}\, \\[6pt]  \,1 & \omega^{-2} \, & \omega^{-4}\,
\end{pmatrix}
\endgroup
\begin{pmatrix} \, \pi_{\text{p}, e}(a)/\psi_e(0)\,  \\[8pt] \, \pi_{\text{p}, e}(a)/\psi_e(1)\,  \\[8pt] \, \pi_{\text{p}, e}(a)/\psi_e(2)\,\end{pmatrix} 
\end{equation}
with $\omega = \mathrm{e}^{-2\pi \mathrm{i}/3}$.

Due to symmetry in factors~(\ref{eqn:PottsPot1}) in the primal NFG and factors~(\ref{eqn:PottsDJ}) in the dual NFG of the 
$q$-state Potts model, it holds that
\begin{equation} 
\label{eqn:PottsSymPrimal}
\frac{\pi_{\text{p}, e}(1)}{\psi_e(1)} = \frac{\pi_{\text{p}, e}(2)}{\psi_e(2)} = \cdots = \frac{\pi_{\text{p}, e}(q-1)}{\psi_e(q-1)}
\end{equation}
and
\begin{equation} 
\label{eqn:PottsSymDual}
\myfracc{\pi_{\text{d}, e}(1)}{\tilde \psi_e(1)} = \myfracc{\pi_{\text{d}, e}(2)}{\tilde \psi_e(2)} = \cdots = \myfracc{\pi_{\text{d}, e}(q-1)}{\tilde \psi_e(q-1)}
\end{equation}
Hence
\begin{equation} 
\label{eqn:PottsMapping0}
\frac{\pi_{\text{p}, e}(0)}{\psi_e(0)} = \sum_{a \in \calA} \myfracc{\pi_{\text{d}, e}(a)}{\tilde \psi_e(a)}
\end{equation}
and for $t \in \{1,2,\ldots, q-1\}$
\begin{align}
\label{eqn:PottsMappingt}
\frac{\pi_{\text{p}, e}(t)}{\psi_e(t)} & = \sum_{a \in \calA} \myfracc{\pi_{\text{d}, e}(a)}{\tilde \psi_e(a)}\hspace{0.07mm}\omega^{at}_{|\calA|} \\
& = \myfracc{\pi_{\text{d}, e}(0)}{\tilde \psi_e(0)} - \myfracc{\pi_{\text{d}, e}(1)}{\tilde \psi_e(1)}
\end{align}
which are real-valued.

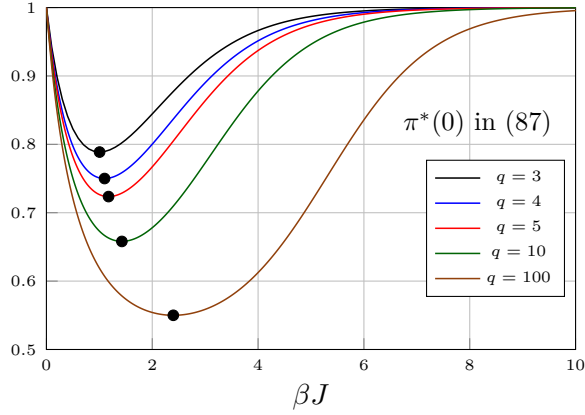
\begin{figure}
\centering
\begin{tikzpicture}
\begin{axis}[
			legend style={at = {(0.98,0.555)} ,font=\tiny},		
			height = 37.0ex,
			width = 52.0ex,
			grid = major,
			tick pos=left, 
			xlabel shift = -2 pt,
			xminorticks = false,	
		    yminorticks = false,	
		    y tick label style={
        /pgf/number format/.cd,
            fixed,
        /tikz/.cd
    		}, 				
			ytick={0.5, 0.6, 0.7, 0.8, 0.9, 1.0},
			xtick={0, 2, 4, 6, 8, 10},
		xlabel= $\beta J$ ={font=\normalsize},
			xmin = 0.0,
			xmax = 10.0,
			ymin = 0.5,
			ymax = 1.0,
			yticklabel style = {font=\tiny,yshift=0.0ex},
            xticklabel style = {font=\tiny,xshift=0.0ex}			
			]

\pgfplotstableread{./files/PPotts3zero.txt}\mydataone
\pgfplotstableread{./files/PPotts4zero.txt}\mydatatwo
\pgfplotstableread{./files/PPotts5zero.txt}\mydatathree
\pgfplotstableread{./files/PPotts10zero.txt}\mydataeight
\pgfplotstableread{./files/PPotts100zero.txt}\mydatathirteen
		
		\addplot [
		line width = 0.58 pt,
		 color = black
		]		
		 table[y = Z] from \mydataone;
		 
		 \addplot [
		 line width = 0.58 pt,
		 color = blue
		 ]
		  table[y = Z] from \mydatatwo;	 
		  
		  \addplot [
		 line width = 0.58 pt,
		 color = red
		 ]
		  table[y = Z] from \mydatathree;	 

%
%
%
		
			  	  \addplot [
		 line width = 0.58 pt,
		 color = chocolate1
		 ]
		  table[y = Z] from \mydataeight;	 		  
%
%
%
%
		  
		  	  	  \addplot [
		 line width = 0.58 pt,
		 color = chocolate2
		 ]
		  table[y = Z] from \mydatathirteen;	 		  	  
		    
		\addplot[mark=*] coordinates {(1.005,0.7887)};
		\addplot[mark=*] coordinates {(1.099,0.75)};
		\addplot[mark=*] coordinates {(1.174,0.7236)};
		\addplot[mark=*] coordinates {(1.426,0.658)};
		\addplot[mark=*] coordinates {(2.398,0.55)};
		
		\node at (8.15, 0.83)   {\small{$\pi^{*}(0)$ in (\ref{eqn:fixedPottsFerzero})}};
%



 \legend{$q=3$, $q=4$, $q=5$, $q=10$, $q=100$};


\end{axis}
\end{tikzpicture}
\caption{\label{fig:L2}%
The fixed points~(\ref{eqn:fixedPottsFerzero}) as a function of 
$\beta J$ for different values of $q$. 
The filled circles show the fixed points at the criticality of the 2D homogeneous Potts model in zero field
located at $\beta J_\mathrm{c} = \ln(1+\sqrt{q})$.}
\end{figure}


\begin{figure}[t]
\centering
\begin{tikzpicture}
\begin{axis}[
			legend style={at = {(0.98,0.78)} ,font=\tiny},		
			height = 37.0ex,
			width = 52.0ex,
			grid = major,
			tick pos=left, 
			xlabel shift = -2 pt,
			xminorticks = false,	
		    yminorticks = false,	
		    y tick label style={
        /pgf/number format/.cd,
            fixed,
        /tikz/.cd
    		}, 				
			ytick={0, 0.02, 0.04, 0.06, 0.08, 0.1, 0.12},
		xlabel= $\beta J$ ={font=\normalsize},
			xmin = 0.0,
			xmax = 10.0,
			ymin = 0.0,
			ymax = 0.12,
			yticklabel style = {font=\tiny,yshift=0.0ex},
            xticklabel style = {font=\tiny,xshift=0.0ex}			
			]

\pgfplotstableread{./files/PPotts3one.txt}\mydataone
\pgfplotstableread{./files/PPotts4one.txt}\mydatatwo
\pgfplotstableread{./files/PPotts5one.txt}\mydatathree
\pgfplotstableread{./files/PPotts10one.txt}\mydataeight
\pgfplotstableread{./files/PPotts100one.txt}\mydatathirteen
		
		\addplot [
		line width = 0.58 pt,
		 color = black
		]		
		 table[y = Z] from \mydataone;
		 
		 \addplot [
		 line width = 0.58 pt,
		 color = blue
		 ]
		  table[y = Z] from \mydatatwo;	 
		  
		  \addplot [
		 line width = 0.58 pt,
		 color = red
		 ]
		  table[y = Z] from \mydatathree;	 

%
%
%
		
			  	  \addplot [
		 line width = 0.58 pt,
		 color = chocolate1
		 ]
		  table[y = Z] from \mydataeight;	 		  
%
%
%
%
		  
		  	  	  \addplot [
		 line width = 0.58 pt,
		 color = chocolate2
		 ]
		  table[y = Z] from \mydatathirteen;

		\addplot[mark=*] coordinates {(1.005,0.10565)};
		\addplot[mark=*] coordinates {(1.099,0.08333)};
		\addplot[mark=*] coordinates {(1.174,0.0691)};
		\addplot[mark=*] coordinates {(1.426,0.038)};
		\addplot[mark=*] coordinates {(2.398,0.004545)};

		\node at (8.15, 0.105)   {\small{$\pi^{*}(1)$ in (\ref{eqn:fixedPottsFerone})}};

  \legend{$q=3$, $q=4$, $q=5$, $q=10$, $q=100$};


\end{axis}
\end{tikzpicture}
\caption{\label{fig:L3}%
The fixed points~(\ref{eqn:fixedPottsFerone}) as a function of 
$\beta J$ for different values of $q$. 
The filled circles show the fixed points at the criticality of the 2D homogeneous Potts model in zero field
located at $\beta J_\mathrm{c} = \ln(1+\sqrt{q})$.}
\end{figure}
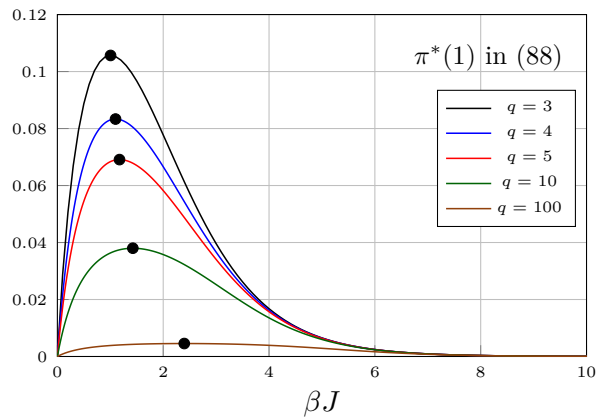

A straightforward generalization of (\ref{eqn:fixedGeneralBinary}) gives the fixed points $(\pi^{*}_{e}(a), a\in \calA)$ of 
the mapping in (\ref{eqn:mappingPottsvector}) as
\begin{equation}
\label{eqn:noname}
\pi^{*}_{e}(a) = \dfrac{ \psi_e(a)\tilde \psi_e(a)}{S}, \quad a \in \calA
\end{equation}
where $S = \sum_{a \in \calA} \psi_e(a)\tilde \psi_e(a)$.

For a homogeneous and ferromagnetic Potts model, the fixed points are given by
\begin{equation}
\label{eqn:fixedPottsFerzero}
\pi^{*}(0) = \dfrac{\textrm{e}^{\beta J}(\textrm{e}^{\beta J} -1 + q)}{\vphantom{2^{2^2}}\textrm{e}^{2\beta J}- 2(1-q)\textrm{e}^{\beta J} + 1 -q}
\end{equation}
and
\begin{equation}
\label{eqn:fixedPottsFerone}
\pi^{*}(t) = \dfrac{\textrm{e}^{\beta J}-1}{\vphantom{2^{2^2}}\textrm{e}^{2\beta J} - 2(1- q)\textrm{e}^{\beta J}+1-q}
\end{equation}
for $t \in \{1, 2, \ldots, q-1\}$. 

Fig.~\ref{fig:L2} shows the fixed point~(\ref{eqn:fixedPottsFerzero}) as a function of $\beta J$ for different values of $q$. 
As in the case of the Ising model, the minimum of $\pi^{*}(0)$ 
is attained at the criticality of the 2D homogeneous
Potts model without an external field, which is 
located at $\beta J_\textrm{c} = \ln(1+\sqrt{q})$.
The fixed points in (\ref{eqn:fixedPottsFerone}) are plotted in Fig.~\ref{fig:L3}. We observe that
the maximum of $(\pi^{*}(t), t \in \{1, 2, \ldots, q-1\})$ are also attained at the criticality.
 For more details on the phase transition of the 2D Potts model, see~\citep{Potts:52, wu1982potts}. 

A closed-form solution for the partition function of the 2D Potts model is not available. However, 
from (\ref{eqn:fixedPottsFerzero}) and (\ref{eqn:fixedPottsFerone}) we can still obtain the values of the fixed points at 
criticality, which are given by
\begin{equation}
\pi^{*}(0) = \frac{1}{2}(1+\frac{1}{\sqrt{q}})
\end{equation}
and
\begin{equation}
\pi^{*}(t) = \frac{1}{2(q-1)}(1-\frac{1}{\sqrt{q}}),
\end{equation}
for $t \in \{1,2, \ldots, q-1\}$. The filled circles in
Figs.~\ref{fig:L2} and~\ref{fig:L3} show the fixed points at criticality of the 2D homogeneous Potts model for different values of $q$.

In summary
\begin{multline}
\label{eqn:fixedpointsPottsAll}
\left . \begin{pmatrix} \, \pi^{*}(0)  & \pi^{*}(1) & \pi^{*}(2) & \cdots & \pi^{*}(q-1)\, \end{pmatrix}\right\vert_{\beta J = \beta J_\text{c}} = \\
\frac{1}{2\sqrt{q}}\begin{pmatrix} \, 1+\sqrt{q} & \dfrac{1}{1+\sqrt{q}} & \dfrac{1}{1+\sqrt{q}} & \cdots & \dfrac{1}{1+\sqrt{q}} \, \end{pmatrix}
\end{multline}


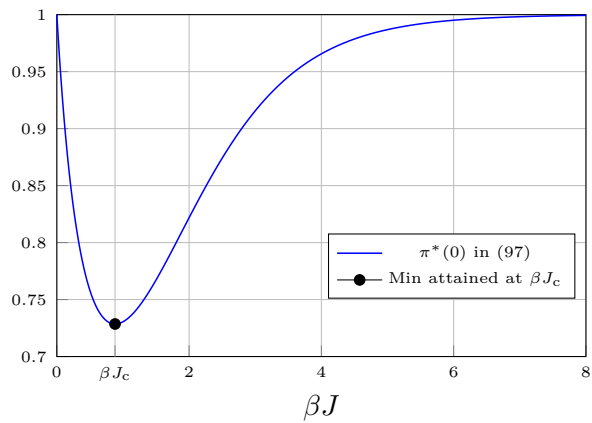
\begin{figure}
\centering
\begin{tikzpicture}
\begin{axis}[
			legend style={at = {(0.98,0.36)} ,font=\tiny},		
			height = 37.0ex,
			width = 52.0ex,
			grid = major,
			tick pos=left, 
			xlabel shift = -2 pt,
			xminorticks = false,	
		    yminorticks = false,	
		    y tick label style={
        /pgf/number format/.cd,
            fixed,
        /tikz/.cd
    		}, 				
			ytick={0.7, 0.75,  0.8, 0.85, 0.9, 0.95, 1.0},
			xtick={0, 2, 4, 6, 8},
			extra x ticks=0.88,
			extra x tick labels={$\beta J_\mathrm{c}$},
		xlabel= $\beta J$ ={font=\normalsize},
			xmin = 0.0,
			xmax = 8.0,
			ymin = 0.7,
			ymax = 1.0,
			yticklabel style = {font=\tiny,yshift=0.0ex},
            xticklabel style = {font=\tiny,xshift=0.0ex}			
			]

\pgfplotstableread{./files/ClockZero.txt}\mydataone

		\addplot [
		line width = 0.58 pt,
		 color = blue
		]		
		 table[y = Z] from \mydataone;
		 
%

		\addplot[mark=*] coordinates {(0.88,0.72855)};
		
		 \legend{$\pi^{*}(0)$ in~(\ref{eqn:fixedClock4}),
 		 Min attained at $\beta J_\mathrm{c}$};

\end{axis}
\end{tikzpicture}
\vspace{-2.0ex}
\caption{\label{fig:LClock1}%
The fixed points of $\pi^{*}(0)$ in~(\ref{eqn:fixedClock4}) as a function of 
$\beta J$. 
The filled circle shows the fixed point at the criticality of the 2D homogeneous four-state clock model, which is
located at $\beta J_\mathrm{c} = \ln(1+\sqrt{2})$.}
\end{figure}
In the many-component limit (i.e., as $q \to \infty$), we obtain
\begin{equation}
\lim_{q \to \infty} \pi^{*}(0) = \frac{1}{2}
\end{equation}

We state without proof that
\begin{proposition}
\label{prop:2DPottsBoundsProp}
In an arbitrary ferromagnetic Potts model in a nonnegative external field, it holds that
\end{proposition}
\begin{equation}
\label{eqn:PottsBoundonPe0}
\pi_{\mathrm{p},e}(0) \ge \dfrac{\textrm{e}^{\beta J_e}}{\textrm{e}^{\beta J_e} - 1 + q }
\end{equation}
\emph{and}
\begin{equation}
\label{eqn:PottsBoundonPd0}
\pi_{\mathrm{d},e}(0) \ge \dfrac{\textrm{e}^{\beta J_e} - 1 + q}{q\textrm{e}^{\beta J_e}},
\end{equation}
\emph{which gives} 
\begin{equation}
\pi_{\mathrm{p},e}(0)\pi_{\mathrm{d},e}(0) \ge \dfrac{1}{q}
\end{equation}

As in the case of the Ising model, the bounds in (\ref{eqn:PottsBoundonPe0}) and (\ref{eqn:PottsBoundonPd0}) intersect at 
the criticality of the 2D homogenous Potts model, cf.~Fig.~\ref{fig:PBound}

We briefly mention analogous results for the 2D homogeneous four-state clock model in zero field, which has a phase transition 
at $\beta J_\mathrm{c} = \ln(1+\sqrt{2})$, see~\citep{kihara1954statistics}.

\begin{figure}
\centering
\begin{tikzpicture}
\begin{axis}[
			legend style={at = {(0.98,0.93)} ,font=\tiny},		
			height = 37.0ex,
			width = 52.0ex,
			grid = major,
			tick pos=left, 
			xlabel shift = -2 pt,
			xminorticks = false,	
		    yminorticks = false,	
		    y tick label style={
        /pgf/number format/.cd,
            fixed,
        /tikz/.cd
    		}, 				
			ytick={0.02, 0.04, 0.06, 0.08, 0.1, 0.12, 0.14},
			xtick={0, 2, 4, 6, 8, 10},
			extra x ticks=0.88,
			extra x tick labels={$\beta J_\mathrm{c}$},
		xlabel= $\beta J$ ={font=\normalsize},
			xmin = 0.0,
			xmax = 8.0,
			ymin = 0.0,
			ymax = 0.14,
			yticklabel style = {font=\tiny,yshift=0.0ex},
            xticklabel style = {font=\tiny,xshift=0.0ex}			
			]

\pgfplotstableread{./files/ClockOne.txt}\mydatatwo
\pgfplotstableread{./files/ClockTwo.txt}\mydatathree

		 \addplot [
		 line width = 0.58 pt,
		 color = black
		 ]
		  table[y = Z] from \mydatatwo;	 
		  
		  \addplot [
		 line width = 0.58 pt,
		 densely dashed,
		 color = blue
		 ]
		  table[y = Z] from \mydatathree;

		\addplot[mark=*] coordinates {(0.88,0.021446)};
		\addplot[mark=*] coordinates {(0.88,0.125)};

%



 		 \legend{$\pi^{*}(1)$ in~(\ref{eqn:fixedClock4}), $\pi^{*}(2)$ in~(\ref{eqn:fixedClock4}), Max  attained at $\beta J_\mathrm{c}$};


\end{axis}
\end{tikzpicture}
\caption{\label{fig:LClock2}%
The fixed points of $\pi^{*}_{\text{p}, e}(1)$ and $\pi^{*}_{\text{p},e}(2)$ in~(\ref{eqn:fixedClock4}) as a function of 
$\beta J$. 
The filled circles show the fixed points at the criticality of the 2D homogeneous four-state clock model
located at $\beta J_\mathrm{c} = \ln(1+\sqrt{2})$.}
\end{figure}

In the clock model
\begin{equation} 
\label{eqn:ClockDual4}
\tilde \psi(\tilde y_e) = \left\{ \begin{array}{ll}
      2(\cosh(\beta J) + 1), & \text{if $\tilde y_e = 0$} \\
      2(\cosh(\beta J) - 1), & \text{if $\tilde y_e = 2$} \\
      2\sinh(\beta J), & \text{otherwise,}
  \end{array} \right.
\end{equation}
which is the 1D DFT of (\ref{eqn:ClockPot1}), and is positive if the model is ferromagnetic (i.e., if $\beta J > 0$).

For this model, the fixed points of the mapping (\ref{eqn:mappingPottsvector}) at criticality are as in
\begin{multline}
\label{eqn:fixedClock4}
\left . \begin{pmatrix} \, \pi^{*}(0)  & \pi^{*}(1)\, & \pi^{*}(2)\, & \pi^{*}(3)\,\end{pmatrix} \right\vert_{\beta J = \beta J_\text{c}}
= \\ \myfrac{1}{4(1+\sinh(\beta J))^2}
\begin{pmatrix} \, (e^{\beta J} + 1)^2 & 2\sinh(\beta J) & (e^{-\beta J} - 1)^2 & 2\sinh(\beta J)\, \end{pmatrix},
\end{multline}
which are plotted as a function of $\beta J$ in Figs.~\ref{fig:LClock1} and~\ref{fig:LClock2}.

The filled circles show the fixed points at the criticality 
of the model, given by
\begin{equation}
\label{eqn:fixedpointsClock4}
\left . \begin{pmatrix} \, \pi^{*}(0)  & \pi^{*}(1) & \pi^{*}(2) & \pi^{*}(3)\, \end{pmatrix}\right\vert_{\beta J = \beta J_\text{c}} =
\begin{pmatrix} \, \dfrac{3+2\sqrt{2}}{8} & \dfrac{1}{8} & \dfrac{3-2\sqrt{2}}{8} & \dfrac{1}{8} \, \end{pmatrix}, 
\end{equation}
which coincide with the minimum of $\pi^{*}(0)$ and the maximums
of $\pi^{*}(1)$ and $\pi^{*}(2)$ in (\ref{eqn:fixedClock4}).

The clock model exhibits Kosterlitz-Thouless transitions for large values of $q$, which is beyond the scope of 
this paper. We refer interested readers to~\citep{kosterlitz1973ordering} and \citep[Chapter 7]{Nishimori:15}.

\section{Numerical Experiments}
\label{sec:NumExp}

In both domains, estimates of marginal densities can be obtained via Markov chain Monte Carlo methods or via 
variational 
algorithms~\citep{christian1999monte, murphy:2012}.
We only consider the subgraphs-world process (SWP) and two variational 
algorithms, the BP and the  TEP algorithms, for ferromagnetic Ising models and frustrated Potts models. 
Estimated marginals in the dual domain are then transformed 
all together to the primal domain via (\ref{eqn:MapDP}). 
In all experiments, the exact values of marginal 
probabilities are computed via the junction tree algorithm implemented in \citep{Mooij:2010}.


In our first experiment, we consider a 2D homogeneous Ising model, in a constant external field $\beta H = 0.15$, with 
periodic boundaries, and with size $N = 6\times6$. For this model, BP and TEP 
in the primal and in the dual domains give virtually indistinguishable approximations. We also apply 
SWP using $10^5$ samples.
\Fig{fig:IsingHom} shows the relative error in estimating $\pi_{\text{p}, e}(0)$ as 
a function of $\beta J$.

We observe that SWP provides good estimates of the marginal probability in the whole range. 
However, compared to variational algorithms, convergence of the SWP is slow; moreover,
SWP is only applicable when the external field is nonzero. In the next two experiments, we consider Ising models in the absence of an
 external field, and only compare the efficiency of variational algorithms employed in 
 the primal and in the dual domains.

In the second experiment, we consider a 2D ferromagnetic Ising model with size $N = 6\times6$, in zero field, and with 
periodic boundaries. Couplings are chosen randomly according
to a half-normal distribution, i.e., $\beta J_e = |\beta J'_e|$ with $\beta J'_e \overset{\text{i.i.d.}}{\sim} \calN(0, \sigma^2)$. 

\Fig{fig:IsingGauss} shows the average relative error in estimating the marginal 
probability $\pi_{\text{p}, e}(0)$ as a function of $\sigma^2$, where the results are averaged over 200 independent realizations.
We consider a fully-connected Ising model with $N = 10$ in our third experiment. Couplings are 
chosen randomly 
according to $\beta J_e \overset{\text{i.i.d.}}{\sim} \calU [0.05, \beta J_x]$, i.e., uniformly between 0.05 
and $\beta J_x$
denoted by the value on the $x$-axis. The average relative error over 50 independent realizations 
is illustrated in~\Fig{fig:IsingFully}.

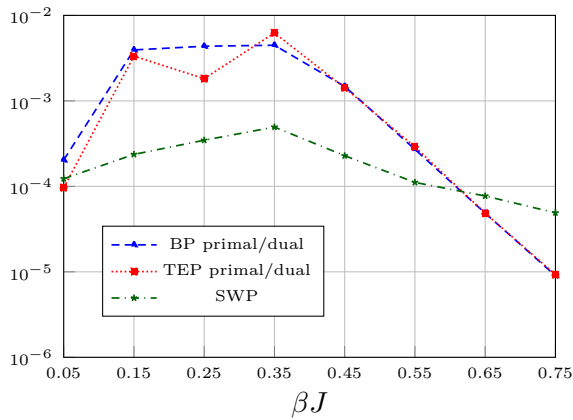
\begin{figure}[t]
\centering
\begin{tikzpicture}
\begin{semilogyaxis}[
			legend style={at = {(0.53,0.385)} ,font=\tiny},		
			height = 37.0ex,
			width = 49.0ex,
			grid = major,
			tick pos=left, 
			xlabel shift = -2 pt,
			xminorticks = false,	
		    yminorticks = false,
    		x tick label style={
        /pgf/number format/.cd,
            fixed,
            fixed zerofill,
            precision=2,
        /tikz/.cd
    		}, 						
			ytick={1e-6, 1e-5, 1e-4, 1e-3, 1e-2, 1e-1},
			yticklabels = {$10^{-6}$, $10^{-5}$, $10^{-4}$, $10^{-3}$, $10^{-2}$},
			xtick={0.05, 0.15, 0.25, 0.35, 0.45, 0.55, 0.65, 0.75, 0.85},
			xlabel= $\beta J$ ={font=\normalsize},
			xmin = 0.05,
			xmax = 0.75,
			ymin = 1e-6,
			ymax = 1e-2,
			yticklabel style = {font=\tiny,yshift=0.0ex},
            xticklabel style = {font=\tiny,xshift=0.0ex}	]

\addplot [mark size=1.3, blue, mark=triangle*, densely dashed, line width = 0.65 pt] table[x={J}, y={BP}] {./Data/IsingREH0.15.txt};
\addplot [mark size=1.2, red, mark=square*, densely dotted, line width = 0.65 pt] table[x={J}, y={TP}] {./Data/IsingDREH0.15.txt};
\addplot [mark size=1.3, chocolate1, mark= star, dashdotted, line width = 0.65 pt] table[x={J}, y={RE}] {./Data/GibbsDual2.txt};

\legend{BP primal/dual, TEP primal/dual, SWP};

\end{semilogyaxis}
\end{tikzpicture}
\vspace{-2.0ex}
\caption{\label{fig:IsingHom}%
Relative error as a function of  $\beta J$ in estimating $\pi_{\text{p}, e}(0)$ of a homogeneous Ising model 
in a constant external field $\beta H = 0.15$, with periodic boundaries, and with size $N = 6\times6$.}
\end{figure}


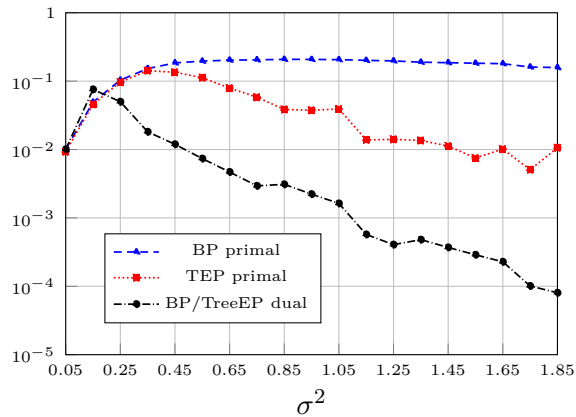
\begin{figure}
\centering
\begin{tikzpicture}
\begin{semilogyaxis}[
			legend style={at = {(0.52,0.355)} ,font=\tiny},		
			height = 37.0ex,
			width = 49.0ex,
			grid = major,
			tick pos=left, 
			xlabel shift = -2 pt,
			xminorticks = false,	
		    yminorticks = false,
    		x tick label style={
        /pgf/number format/.cd,
            fixed,
            fixed zerofill,
            precision=2,
        /tikz/.cd
    		}, 						
			ytick={1e-5, 1e-4, 1e-3, 1e-2, 1e-1, 1e0},
			yticklabels = {$10^{-5}$, $10^{-4}$, $10^{-3}$, $10^{-2}$, $10^{-1}$, $1$},
			xtick={0.05, 0.25, 0.45, 0.65, 0.85, 1.05, 1.25, 1.45, 1.65, 1.85},
			xlabel= $\sigma^2$ ={font=\normalsize},
			xmin = 0.05,
			xmax = 1.85,
			ymin = 1e-5,
			ymax = 1e0,
			yticklabel style = {font=\tiny,yshift=0.0ex},
            xticklabel style = {font=\tiny,xshift=0.0ex}	]

\addplot [mark size=1.3, blue, mark=triangle*, densely dashed, line width = 0.65 pt] table[x={J}, y={BP}] {./Data/RelativeErrorPIsingGauss.txt};
\addplot [mark size=1.2, red, mark=square*, densely dotted, line width = 0.65 pt] table[x={J}, y={TP}] {./Data/RelativeErrorPIsingGauss.txt};
\addplot [mark size=1.2, black, mark=*, densely dashdotted, line width = 0.65 pt] table[x={J}, y={BP}] {./Data/RelativeErrorDIsingGauss.txt};

\legend{BP primal, TEP primal, BP/TreeEP dual};

\end{semilogyaxis}
\end{tikzpicture}
\vspace{-2.0ex}
\caption{\label{fig:IsingGauss}%
Average relative error in estimating $\pi_{\text{p}, e}(0)$ of an Ising model with 
periodic boundaries and with size $N = 6\times6$. Couplings are chosen randomly according
to a half-normal distribution with variance $\sigma^2$.}
\end{figure}


\begin{figure}[h!!!]
\centering
\begin{tikzpicture}
\begin{semilogyaxis}[
			legend style={at = {(0.52,0.36)} ,font=\tiny},		
			height = 37.0ex,
			width = 49.0ex,
			grid = major,
			tick pos=left, 
			xlabel shift = -2 pt,
			xminorticks = false,	
		    yminorticks = false,
    		x tick label style={
        /pgf/number format/.cd,
            fixed,
            fixed zerofill,
            precision=2,
        /tikz/.cd
    		}, 						
			ytick={1e-5, 1e-4, 1e-3, 1e-2, 1e-1, 1e0},
			yticklabels = {$10^{-5}$, $10^{-4}$, $10^{-3}$, $10^{-2}$, $10^{-1}$, $1$},
			xtick={0.05, 0.15, 0.25, 0.35, 0.45, 0.55, 0.65},
			xlabel= $\beta J_x$ ={font=\normalsize},
			xmin = 0.05,
			xmax = 0.65,
			ymin = 1e-5,
			ymax = 1e0,
			yticklabel style = {font=\tiny,yshift=0.0ex},
            xticklabel style = {font=\tiny,xshift=0.0ex}	]

\addplot [mark size=1.3, blue, mark=triangle*, densely dashed, line width = 0.65 pt] table[x={J}, y={BP}] {./Data/RelativeErrorPFully.txt};
\addplot [mark size=1.2, red, mark=square*, densely dotted, line width = 0.65 pt] table[x={J}, y={TP}] {./Data/RelativeErrorPFully.txt};
\addplot [mark size=1.2, black, mark=*, densely dashdotted, line width = 0.65 pt] table[x={J}, y={BP}] {./Data/RelativeErrorDFully.txt};

\legend{BP primal, TEP primal, BP/TreeEP dual};

\end{semilogyaxis}
\end{tikzpicture}
\vspace{-2.0ex}
\caption{\label{fig:IsingFully}%
Average relative error in estimating $\pi_{\text{p}, e}(0)$ in a fully-connected 
Ising model with $N = 10$. Coupling parameters are chosen uniformly and independently between 0.05 and $\beta J_x$ 
denoted by the $x$-axis.}
\vspace{2.0ex}
\centering
\begin{tikzpicture}
\begin{semilogyaxis}[
			legend style={at = {(0.92,0.32)} ,font=\tiny},		
			height = 37.0ex,
			width = 49.0ex,
			grid = major,
			tick pos=left, 
			xlabel shift = -2 pt,
			xminorticks = false,	
		    yminorticks = false,
    		x tick label style={
        /pgf/number format/.cd,
            fixed,
            fixed zerofill,
            precision=2,
        /tikz/.cd
    		}, 						
			ytick={1e-4, 1e-3, 1e-2, 1e-1, 1e0},
			yticklabels = {$10^{-4}$, $10^{-3}$, $10^{-2}$, $10^{-1}$, $1$},
			xtick={0.15, 0.45, 0.75, 1.05, 1.35, 1.65, 1.95, 2.25, 2.55, 2.85},
			xlabel= $\beta J_{e_{\text{Ferr}}}$ ={font=\normalsize},
			xmin = 0.15,
			xmax = 2.85,
			ymin = 1e-4,
			ymax = 1e0,
			yticklabel style = {font=\tiny,yshift=0.0ex},
            xticklabel style = {font=\tiny,xshift=0.0ex}	]

\addplot [mark size=1.3, blue, mark=triangle*, densely dashed, line width = 0.65 pt] table[x={J}, y={BP}] {./DataPottsSpin/PottsRE24ZeroS.txt};
\addplot [mark size=1.1, red, mark=square*, densely dotted, line width = 0.65 pt] table[x={J}, y={TP}] {./DataPottsSpin/PottsRE24ZeroS.txt};
\addplot [mark size=1.2, black, mark=*, densely dashdotted, line width = 0.65 pt] table[x={J}, y={BP}] {./DataPottsSpin/PottsDRE24ZeroS.txt};

\legend{BP primal, TEP primal, BP dual};

\end{semilogyaxis}
\end{tikzpicture}
\vspace{-2.0ex}
\caption{\label{fig:PottsFrus1}%
Relative error in estimating the edge marginal function 
of an edge with ferromagnetic interaction of a frustrated \mbox{3-state} 
Potts model with free boundaries and size $N = 6\times6$. 
Here $e_{\text{Ferr}}$ has a ferromagnetic interaction $\beta J_{e_{\text{Ferr}}}$.}
\end{figure}

In the second and the third experiments, BP and TEP provide
close approximations in the dual domain, therefore only 
BP results are reported.
Figs.~\ref{fig:IsingGauss} and~\ref{fig:IsingFully} show that for 
$\sigma^2 > 0.25$ and $\beta J_x > 0.20$ (i.e., in relatively lower temeratures), BP in the dual NFG can significantly improve the quality of 
estimates -- even by more than two orders of magnitude in terms of relative error. 

In our last experiment, we consider a 2D 3-state Potts model with size $N = 6\times6$, in the absence of an external field, and 
with free boundary conditions, 
in which all plaquettes (i.e., cycles of length four) are frustrated. 

A plaquette is called frustrated if the product of four coupling parameters along 
its edges is negative. It is then not possible to satisfy all local constraints at the same time, which leads to difficult
energy landscapes \citep{Nishimori:15}. 
We then focus on the plaquette in the middle of the model.

In order to create frustration, each plaquette has one coupling parameter set 
to $\beta J_{e_{\text{Antif}}} = -0.25$ (i.e., with antiferromagnetic interaction), and three 
remaining couplings equal to $\beta J_{e_{\text{Ferr}}}$ according to the value on 
the $x$-axis (i.e., with ferromagnetic interaction). 
According to~(\ref{eqn:PottsDJ}), in this example factors with antiferromagnetic interactions will have negative 
components in the dual NFG, c.f.~Remark 4.

\Fig{fig:PottsFrus1} shows the relative error in estimating the edge marginal function 
of an edge with ferromagnetic interaction as a function of $\beta J_{ e_{\text{Ferr}}}$. 
In this example, BP in the dual domain
provides the most accurate estimates when $\beta J_{e_{\text{Ferr}}} > 1.90$. For smaller values 
of $\beta J_{e_{\text{Ferr}}}$, BP in the primal and BP in the dual domains perform similarly, and 
for $\beta J_{e_{\text{Ferr}}} < 1.0$, TEP in the primal domain gives the best estimates. 

In general, we observe that applying inference algorithms on the dual NFG is advantageous when 
the coupling parameters are large (i.e., when the temperature is low).

\section{Extensions to Continuous Models}
\label{sec:Continuous}

As an extension of the proposed mappings to continuous models, we consider the following probability density 
function (PDF) in the primal domain
\begin{align} 
f_\text{p}(\x) & \propto
\textrm{exp}\Big(-\frac{1}{2s^2}\sum_{(k,\ell) \in \EE}(x_k - x_\ell)^2\Big)
\textrm{exp}\Big(-\frac{1}{2\sigma^2}\sum_{v \in \VV} x_v^2\Big) \\
& = \textrm{exp}\Big(-\frac{1}{2s^2}\sum_{e \in \EE}y_e^2\Big)
\textrm{exp}\Big(-\frac{1}{2\sigma^2}\sum_{v \in \VV} x_v^2\Big), \label{eqn:ProbPGaussThin}
\end{align}
where $\x \in \R^{|\VV|}$, $s^2$ denote the intervariable variance, and $\sigma^2$ denotes the vertex variance.
This Gaussian Markov random field with the thin-membrane prior is widely used in Bayesian 
image analysis, in data interpolation, and in learning the structure of Gaussian graphical 
models~\citep{winkler1995, weiss2000correctness, malioutov2008approximate, yu2022efficient}.

The PDF in~(\ref{eqn:ProbPGaussThin}) is in accordance with 
the formulation in Section~\ref{sec:Ising}, as the edge-weighing factor $\psi_{e}(\cdot)$ is only a 
function of the edge configuration $y_e$. Indeed from~(\ref{eqn:ProbP}), we note that
\begin{equation} 
\label{eqn:ProbPGauss2}
\psi_e(y_e) = \textrm{exp}\big(-\frac{1}{2s^2}y_e^2\big)
\end{equation}
and the vertex-weighing factors are given by
\begin{equation} 
\label{eqn:ProbPGauss3}
\phi_v(x_v) = \textrm{exp}\big(-\frac{1}{2\sigma^2} x_v^2\big)
\end{equation}

For a chain graph, the primal NFG of~(\ref{eqn:ProbPGaussThin}) is illustrated in~\Fig{fig:2DGridMod},
where the big unlabeled boxes represent~(\ref{eqn:ProbPGauss2}), the small unlabeled boxes 
represent~(\ref{eqn:ProbPGauss3}), boxes labeled ``$=$'' are instances of equality 
indicator factors $\delta_{=}(\cdot)$, where $\delta_{=}(\cdot)$ denotes the Dirac delta function, 
and boxes labeled ``$+$'' are instances of zero-sum indicator 
factors, which impose the
constraint that all their incident variables sum 
to zero in $\R$. 

In this framework $(\mathcal{F}\psi_e)(\cdot)$, the 1D Fourier transform of $\psi_e(\cdot)$, is the function
$\tilde \psi_e(\cdot)$ which is given by
\begin{align}
\tilde \psi_e(\tilde y_e) & = (\mathcal{F}\psi_e)(\tilde y_e) \\
& = \int_{-\infty}^{\infty} \psi_e(y_e)\textrm{e}^{-\mathrm{i}\,\tilde y_ey_e}dy_e \label{eqn:ft}
\end{align}
where $\mathrm{i}$ is the unit imaginary number~\citep{Brace:99}.

In analogy with (\ref{eqn:Pd}), the PDF in the dual domain can be expressed as
\begin{equation}
\label{eqn:PdGauss1}
f_\text{d}(\tilde \y) \propto \prod_{e \in \EE} \tilde \psi_e(\tilde y_e) \prod_{v \in \VV} \tilde \phi_v(\tilde x_v),
\end{equation} 
where
\begin{align}
\tilde \psi_e(\tilde y_e) & = \sqrt{2\pi s^2}\textrm{exp}\big(-\frac{s^2}{2}\tilde y_e^2\big) \\
& \propto \textrm{exp}\big(-\frac{s^2}{2}\tilde y_e^2\big) \label{eqn:PdGauss2}
\end{align} 
and
\begin{equation}
\label{eqn:PdGauss3}
\tilde \phi_v(\tilde x_v) = \textrm{exp}\big(-\frac{\sigma^2}{2}\tilde  x_v^2\big)
\end{equation} 
are the Fourier transforms of (\ref{eqn:ProbPGauss2}) and (\ref{eqn:ProbPGauss3}) respectively.\footnote{Following~(\ref{eqn:ft}), the Fourier transform 
of $f(x) = \frac{1}{\sqrt{2\pi \sigma^2}}\textrm{exp}\big(-\frac{1}{2\sigma^2} x^2\big)$, i.e., the PDF of a Gaussian distribution with mean zero and with standard 
deviation $\sigma$, is given by $(\mathcal{F}f)(\tilde x) = \textrm{exp}\big(-\frac{\sigma^2}{2} \tilde x^2\big)$. However, the scale factors that appear in the 
Fourier transforms of (\ref{eqn:ProbPGauss2}) and (\ref{eqn:ProbPGauss3}) are unimportant in our formulation. After all,  we are mainly concerned with 
computing the marginal densities of~(\ref{eqn:ProbPGaussThin}) from the corresponding marginal densities of~(\ref{eqn:PdGauss1}).}

The corresponding dual NFG is shown in~\Fig{fig:2DGridModDual}, where factors~(\ref{eqn:PdGauss2}) are represented by the big unlabeled boxes
and factors~(\ref{eqn:PdGauss3}) are represented by the small unlabeled boxes.

A straightforward generalization of Proposition~\ref{prop:EdgeDFT} gives the following local mapping that relates the edge marginal probabilities of the 
primal NFG to the edge marginal probabilities in the dual NFG
\begin{align}
\frac{f_{\text{p}, e}(y_e)}{\psi_e(y_e)} & = \Big(\mathcal{F}\frac{f_{\text{d}, e}}{(\mathcal{F}\psi_e)}\Big)(\tilde y_e) \\
& = \Big(\mathcal{F}\frac{f_{\text{d}, e}}{\tilde \psi_e}\Big)(\tilde y_e) \label{eqn:TransMargsGauss2}
\end{align}
In other words, the functions $f_{\mathrm{p}, e}(t)/\psi_e(t), t\in \R$ and 
$f_{\mathrm{d}, e}(\tilde t)/\tilde \psi_e(\tilde t), \tilde t \in \R$ are Fourier pairs.
Similarly, it holds that
\begin{equation}
\frac{f_{\text{p}, v}(x_v)}{\phi_v(x_v)} = \Big(\mathcal{F}\frac{f_{\text{d}, v}}{\tilde \phi_v}\Big)(\tilde x_v) \label{eqn:TransMargsGauss3}
\end{equation}
as a generalization of Proposition~\ref{prop:VertexDFT}.

Again, as a consequence of (\ref{eqn:TransMargsGauss2}) and (\ref{eqn:TransMargsGauss3}), it is possible to estimate
the edge/vertex marginal densities in the dual domain, and then transform them to the primal domain. Let us denote the 
estimated marginal variance in the dual domain by $\hat\sigma_\text{d}^2$, which can be mapped to the primal 
NFG via (\ref{eqn:TransMargsGauss3}). Indeed
\begin{align}
f_{\text{p}, v}(x_v) & = \phi_v(x_v) \Big(\mathcal{F}\frac{f_{\text{d}, v}}{\tilde \phi_v}\Big)(\tilde x_v) \\ 
& \propto \textrm{exp}\big(-\dfrac{1}{2\sigma^2}x_v^2\big)
\displaystyle\int_{-\infty}^{\infty}\frac{\textrm{exp}
\big(-\dfrac{1}{2\hat\sigma_\text{d}^2}\tilde x_v^2\big)}{\textrm{exp}\big(-\dfrac{\sigma^2}{2}\tilde x_v^2\big)})\textrm{e}^{-\mathrm{i}\,\tilde x_vx_v}d\tilde x_v\\
& \propto \textrm{exp}\big(-\dfrac{1}{2\sigma^2(1-\sigma^2\hat\sigma_\text{d}^2)}x_v^2\big) \label{eqn:TransMargsGauss4}
\end{align}

Let $f_{\text{p}, v}(x_v) \propto \textrm{exp}\big(-\dfrac{1}{2\hat\sigma_{\text{d} \rightarrow \text{p}}^2}x_v^2\big)$, where 
$\hat\sigma_{\text{d} \rightarrow \text{p}}^2$ is the estimated
marginal variance in the primal domain. After matching the exponents on the left-hand side and the-right hand side 
of (\ref{eqn:TransMargsGauss4}), we get
\begin{equation}
\label{eqn:TransVarGauss}
\hat\sigma_{\text{d} \rightarrow \text{p}}^2 = \sigma^2(1-\sigma^2\hat\sigma_\text{d}^2),
\end{equation}
which can be used to compute an estimate of the marginal variance in the primal domain from the estimated marginal variance in the dual domain.


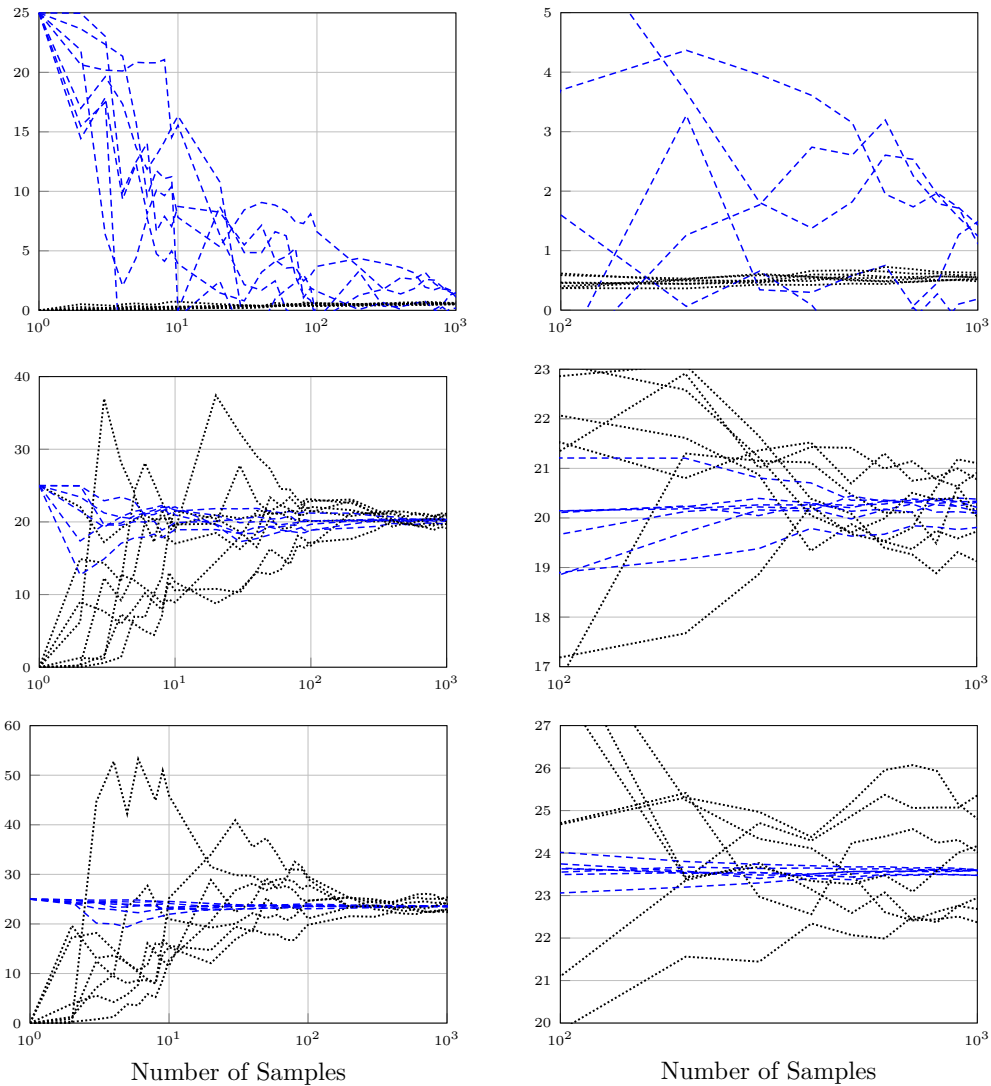
\begin{figure}[t]
  \centering
  \subfigure{
     \centering
\begin{tikzpicture}[scale=0.87]
\begin{axis}[
			legend style={at = {(0.98,0.555)} ,font=\tiny},		
			height = 37.0ex,
			width = 48.0ex,
			grid = major,
			tick pos=left, 
			xlabel shift = -2 pt,
			xmode=log,
			xminorticks = false,	
		    yminorticks = false,	
		    y tick label style={
        /pgf/number format/.cd,
            fixed,
        /tikz/.cd
    		}, 				
			ytick={0, 5, 10, 15, 20, 25},
			xmin = 1.0,
			xmax = 1000,
			ymin = 0.0,
			ymax = 25,
			yticklabel style = {font=\tiny,yshift=0.0ex},
            xticklabel style = {font=\tiny,xshift=0.0ex}			
			]

\pgfplotstableread{./Data/1S5Sigma1.txt}\mydataone
\pgfplotstableread{./Data/2S5Sigma1.txt}\mydatatwo
\pgfplotstableread{./Data/3S5Sigma1.txt}\mydatathree
\pgfplotstableread{./Data/4S5Sigma1.txt}\mydatafour
\pgfplotstableread{./Data/5S5Sigma1.txt}\mydatafive
\pgfplotstableread{./Data/6S5Sigma1.txt}\mydatafourteen
\pgfplotstableread{./Data/7S5Sigma1.txt}\mydatafifteen
\pgfplotstableread{./DataDual/1S5Sigma1.txt}\mydatasix
\pgfplotstableread{./DataDual/2S5Sigma1.txt}\mydataseven
\pgfplotstableread{./DataDual/3S5Sigma1.txt}\mydataeight
\pgfplotstableread{./DataDual/4S5Sigma1.txt}\mydatanine
\pgfplotstableread{./DataDual/5S5Sigma1.txt}\mydataten
\pgfplotstableread{./DataDual/6S5Sigma1.txt}\mydataeleven
\pgfplotstableread{./DataDual/7S5Sigma1.txt}\mydatatwelve
		
		\addplot [
		densely dotted,
		line width = 0.8 pt,
		 color = black
		]		
		 table[y = Z] from \mydataone;
		 
		 \addplot [
		 		densely dotted,
		 line width = 0.8 pt,
		 color = black
		 ]
		  table[y = Z] from \mydatatwo;	 
		  
		  \addplot [
		  		densely dotted,
		 line width = 0.8 pt,
		 color = black
		 ]
		  table[y = Z] from \mydatathree;	 

	  \addplot [
	  		densely dotted,
		 line width = 0.8 pt,
		 color = black
		 ]
		  table[y = Z] from \mydatafour;	 
		  
	  \addplot [
	  		densely dotted,
		 line width = 0.8 pt,
		 color = black
		 ]
		  table[y = Z] from \mydatafive;	 		  
		  
	  \addplot [
		 line width = 0.65 pt,
		 color = blue,
		 		 densely dashed
		 ]
		  table[y = Z] from \mydatasix;	 		  
		  
		  	  \addplot [
		 line width = 0.65 pt,
		 color = blue,
		 		 densely dashed
		 ]
		  table[y = Z] from \mydataseven;	 		  
		
			  	  \addplot [
		 line width = 0.65 pt,
		 color = blue,
		 		 densely dashed
		 ]
		  table[y = Z] from \mydataeight;	 		  
		  
		  	  	  \addplot [
		 line width = 0.65 pt,
		 color = blue,
		 		 densely dashed
		 ]
		  table[y = Z] from \mydatanine;	 		  
		  
		  	  	  \addplot [
		 line width = 0.65 pt,
		 color = blue,
		 		 densely dashed
		 ]
		  table[y = Z] from \mydataten;	 		  
		  
		  	  	  \addplot [
		 line width = 0.65 pt,
		 color = blue,
		 		 densely dashed
		 ]
		  table[y = Z] from \mydataeleven;	 	

		  	  	  \addplot [
		 line width = 0.65 pt,
		 color = blue,
		 		 densely dashed
		 ]
		  table[y = Z] from \mydatatwelve;	 		  
		  
		  		  	  	
		\addplot [
				densely dotted,
		 line width = 0.8 pt,
		 color = black
		 ]
		  table[y = Z] from \mydatafourteen;	 	

		\addplot [		  
				densely dotted,
		 line width = 0.8 pt,
		 color = black
		 ]
		  table[y = Z] from \mydatafifteen;	 		  	  
		    
\end{axis}
\end{tikzpicture}
}
\quad
  \subfigure{
     \centering
\begin{tikzpicture}[scale=0.87]
\begin{axis}[
			legend style={at = {(0.98,0.555)} ,font=\tiny},		
			height = 37.0ex,
			width = 48.0ex,
			grid = major,
			tick pos=left, 
			xlabel shift = -2 pt,
			xmode=log,
			xminorticks = false,	
		    yminorticks = false,	
		    y tick label style={
        /pgf/number format/.cd,
            fixed,
        /tikz/.cd
    		}, 				
			ytick={0, 1, 2, 3, 4, 5},
			xtick={100, 1000},
			xmin = 100,
			xmax = 1000,
			ymin = 0.0,
			ymax = 5,
			yticklabel style = {font=\tiny,yshift=0.0ex},
            xticklabel style = {font=\tiny,xshift=0.0ex}			
			]

\pgfplotstableread{./Data/1S5Sigma1.txt}\mydataone
\pgfplotstableread{./Data/2S5Sigma1.txt}\mydatatwo
\pgfplotstableread{./Data/3S5Sigma1.txt}\mydatathree
\pgfplotstableread{./Data/4S5Sigma1.txt}\mydatafour
\pgfplotstableread{./Data/5S5Sigma1.txt}\mydatafive
\pgfplotstableread{./Data/6S5Sigma1.txt}\mydatafourteen
\pgfplotstableread{./Data/7S5Sigma1.txt}\mydatafifteen
\pgfplotstableread{./DataDual/1S5Sigma1.txt}\mydatasix
\pgfplotstableread{./DataDual/2S5Sigma1.txt}\mydataseven
\pgfplotstableread{./DataDual/3S5Sigma1.txt}\mydataeight
\pgfplotstableread{./DataDual/4S5Sigma1.txt}\mydatanine
\pgfplotstableread{./DataDual/5S5Sigma1.txt}\mydataten
\pgfplotstableread{./DataDual/6S5Sigma1.txt}\mydataeleven
\pgfplotstableread{./DataDual/7S5Sigma1.txt}\mydatatwelve
		
		\addplot [
		densely dotted,
		line width = 0.8 pt,
		 color = black
		]		
		 table[y = Z] from \mydataone;
		 
		 \addplot [
		 densely dotted,
		 line width = 0.8pt,
		 color = black
		 ]
		  table[y = Z] from \mydatatwo;	 
		  
		  \addplot [
		  		densely dotted,
		 line width = 0.8 pt,
		 color = black
		 ]
		  table[y = Z] from \mydatathree;	 

	  \addplot [
	  		densely dotted,
		 line width = 0.8pt,
		 color = black
		 ]
		  table[y = Z] from \mydatafour;	 
		  
	  \addplot [
	  		densely dotted,
		 line width = 0.8 pt,
		 color = black
		 ]
		  table[y = Z] from \mydatafive;	 		  
		  
	  \addplot [
		 line width = 0.65 pt,
		 color = blue,
		 		 densely dashed
		 ]
		  table[y = Z] from \mydatasix;	 		  
		  
		  	  \addplot [
		 line width = 0.65 pt,
		 color = blue,
		 		 densely dashed
		 ]
		  table[y = Z] from \mydataseven;	 		  
		
			  	  \addplot [
		 line width = 0.65 pt,
		 color = blue,
		 		 densely dashed
		 ]
		  table[y = Z] from \mydataeight;	 		  
		  
		  	  	  \addplot [
		 line width = 0.65 pt,
		 color = blue,
		 		 densely dashed
		 ]
		  table[y = Z] from \mydatanine;	 		  
		  
		  	  	  \addplot [
		 line width = 0.65 pt,
		 color = blue,
		 		 densely dashed
		 ]
		  table[y = Z] from \mydataten;	 		  
		  
		  	  	  \addplot [
		 line width = 0.65 pt,
		 color = blue,
		 		 densely dashed
		 ]
		  table[y = Z] from \mydataeleven;	 	

		  	  	  \addplot [
		 line width = 0.65 pt,
		 color = blue,
		 		 densely dashed
		 ]
		  table[y = Z] from \mydatatwelve;	 		  
%
		  		  	  	
		\addplot [
				densely dotted,
		 line width = 0.8 pt,
		 color = black
		 ]
		  table[y = Z] from \mydatafourteen;	 	

		\addplot [		
				densely dotted,  
		 line width = 0.8 pt,
		 color = black
		 ]
		  table[y = Z] from \mydatafifteen;	 		  	  
		    
\end{axis}
\end{tikzpicture}
}
\centering
  \subfigure{
       \centering
\begin{tikzpicture}[scale=0.85]
\begin{axis}[
			legend style={at = {(0.98,0.555)} ,font=\tiny},		
			height = 37.0ex,
			width = 48.0ex,
			grid = major,
			tick pos=left, 
			xlabel shift = -2 pt,
			xmode=log,
			xminorticks = false,	
		    yminorticks = false,	
		    y tick label style={
        /pgf/number format/.cd,
            fixed,
        /tikz/.cd
    		}, 				
			xmin = 1.0,
			xmax = 1000,
			ymin = 0.0,
			ymax = 40.0,
			yticklabel style = {font=\tiny,yshift=0.0ex},
            xticklabel style = {font=\tiny,xshift=0.0ex}			
			]
\pgfplotstableread{./Data/1S5Sigma20.txt}\mydataone
\pgfplotstableread{./Data/2S5Sigma20.txt}\mydatatwo
\pgfplotstableread{./Data/3S5Sigma20.txt}\mydatathree
\pgfplotstableread{./Data/4S5Sigma20.txt}\mydatafour
\pgfplotstableread{./Data/5S5Sigma20.txt}\mydatafive
\pgfplotstableread{./Data/6S5Sigma20.txt}\mydatafourteen
\pgfplotstableread{./Data/7S5Sigma20.txt}\mydatafifteen
\pgfplotstableread{./DataDual/1S5Sigma20.txt}\mydatasix
\pgfplotstableread{./DataDual/2S5Sigma20.txt}\mydataseven
\pgfplotstableread{./DataDual/3S5Sigma20.txt}\mydataeight
\pgfplotstableread{./DataDual/4S5Sigma20.txt}\mydatanine
\pgfplotstableread{./DataDual/5S5Sigma20.txt}\mydataten
\pgfplotstableread{./DataDual/6S5Sigma20.txt}\mydataeleven
\pgfplotstableread{./DataDual/7S5Sigma20.txt}\mydatatwelve
		
		\addplot [
		densely dotted,  
		 line width = 0.85 pt,
		 color = black
		]		
		 table[y = Z] from \mydataone;
		 
		 \addplot [
		densely dotted,  
		 line width = 0.85 pt,
		 color = black
		 ]
		  table[y = Z] from \mydatatwo;	 
		  
		  \addplot [
		densely dotted,  
		 line width = 0.85 pt,
		 color = black
		 ]
		  table[y = Z] from \mydatathree;	 

	  \addplot [
		densely dotted,  
		 line width = 0.85 pt,
		 color = black
		 ]
		  table[y = Z] from \mydatafour;	 
		  
	  \addplot [
		densely dotted,  
		 line width = 0.85 pt,
		 color = black
		 ]
		  table[y = Z] from \mydatafive;	 		  
		  
	  \addplot [
		densely dotted,  
		 line width = 0.85 pt,
		 color = black
		 ]
		  table[y = Z] from \mydatasix;	 		  
		  
		  	  \addplot [
		 line width = 0.65 pt,
		 color = blue,
		 		 		 densely dashed
		 ]
		  table[y = Z] from \mydataseven;	 		  
		
			  	  \addplot [
		 line width = 0.65 pt,
		 color = blue,
		 		 		 densely dashed
		 ]
		  table[y = Z] from \mydataeight;	 		  
		  
		  	  	  \addplot [
		 line width = 0.65 pt,
		 color = blue,
		 		 		 densely dashed
		 ]
		  table[y = Z] from \mydatanine;	 		  
		  
		  	  	  \addplot [
		 line width = 0.65 pt,
		 color = blue,
		 		 		 densely dashed
		 ]
		  table[y = Z] from \mydataten;	 		  
		  
		  	  	  \addplot [
		 line width = 0.65 pt,
		 color = blue,
		 		 		 densely dashed
		 ]
		  table[y = Z] from \mydataeleven;	 	

		  	  	  \addplot [
		 line width = 0.65 pt,
		 color = blue,
		 		 		 densely dashed
		 ]
		  table[y = Z] from \mydatatwelve;	 		  
		  
		  		  	  	
		\addplot [
		densely dotted,  
		 line width = 0.85 pt,
		 color = black
		 ]
		  table[y = Z] from \mydatafourteen;	 	

		\addplot [		  
		densely dotted,  
		 line width = 0.85 pt,
		 color = black
		 ]
		  table[y = Z] from \mydatafifteen;	 		  	  
		    
\end{axis}
\end{tikzpicture}
}
\quad
  \subfigure{
       \centering
\begin{tikzpicture}[scale=0.87]
\begin{axis}[
			legend style={at = {(0.98,0.555)} ,font=\tiny},		
			height = 37.0ex,
			width = 48.0ex,
			grid = major,
			tick pos=left, 
			xlabel shift = -2 pt,
			xmode=log,
			xminorticks = false,	
		    yminorticks = false,	
		    y tick label style={
        /pgf/number format/.cd,
            fixed,
        /tikz/.cd
    		}, 				
			ytick={17, 18, 19, 20, 21, 22, 23, 24},
			xtick={100, 1000},
			xmin = 100,
			xmax = 1000,
			ymin = 17,
			ymax = 23,
			yticklabel style = {font=\tiny,yshift=0.0ex},
            xticklabel style = {font=\tiny,xshift=0.0ex}			
			]
\pgfplotstableread{./Data/1S5Sigma20.txt}\mydataone
\pgfplotstableread{./Data/2S5Sigma20.txt}\mydatatwo
\pgfplotstableread{./Data/3S5Sigma20.txt}\mydatathree
\pgfplotstableread{./Data/4S5Sigma20.txt}\mydatafour
\pgfplotstableread{./Data/5S5Sigma20.txt}\mydatafive
\pgfplotstableread{./Data/6S5Sigma20.txt}\mydatafourteen
\pgfplotstableread{./Data/7S5Sigma20.txt}\mydatafifteen
\pgfplotstableread{./DataDual/1S5Sigma20.txt}\mydatasix
\pgfplotstableread{./DataDual/2S5Sigma20.txt}\mydataseven
\pgfplotstableread{./DataDual/3S5Sigma20.txt}\mydataeight
\pgfplotstableread{./DataDual/4S5Sigma20.txt}\mydatanine
\pgfplotstableread{./DataDual/5S5Sigma20.txt}\mydataten
\pgfplotstableread{./DataDual/6S5Sigma20.txt}\mydataeleven
\pgfplotstableread{./DataDual/7S5Sigma20.txt}\mydatatwelve
		
		\addplot [
		densely dotted,  
		 line width = 0.85 pt,
		 color = black
		]		
		 table[y = Z] from \mydataone;
		 
		 \addplot [
		densely dotted,  
		 line width = 0.85 pt,
		 color = black
		 ]
		  table[y = Z] from \mydatatwo;	 
		  
		  \addplot [
		densely dotted,  
		 line width = 0.85 pt,
		 color = black
		 ]
		  table[y = Z] from \mydatathree;	 

	  \addplot [
		densely dotted,  
		 line width = 0.85 pt,
		 color = black
		 ]
		  table[y = Z] from \mydatafour;	 
		  
	  \addplot [
		densely dotted,  
		 line width = 0.85 pt,
		 color = black
		 ]
		  table[y = Z] from \mydatafive;	 		  
		  
	  \addplot [
		 line width = 0.65 pt,
		 color = blue,
		 		 		 densely dashed
		 ]
		  table[y = Z] from \mydatasix;	 		  
		  
		  	  \addplot [
		 line width = 0.65 pt,
		 color = blue,
		 		 		 densely dashed
		 ]
		  table[y = Z] from \mydataseven;	 		  
		
			  	  \addplot [
		 line width = 0.65 pt,
		 color = blue,
		 		 		 densely dashed
		 ]
		  table[y = Z] from \mydataeight;	 		  
		  
		  	  	  \addplot [
		 line width = 0.65 pt,
		 color = blue,
		 		 		 densely dashed
		 ]
		  table[y = Z] from \mydatanine;	 		  
		  
		  	  	  \addplot [
		 line width = 0.65 pt,
		 color = blue,
		 		 		 densely dashed
		 ]
		  table[y = Z] from \mydataten;	 		  
		  
		  	  	  \addplot [
		 line width = 0.65 pt,
		 color = blue,
		 		 		 densely dashed
		 ]
		  table[y = Z] from \mydataeleven;	 	

		  	  	  \addplot [
		 line width = 0.65 pt,
		 color = blue,
		 densely dashed
		 ]
		  table[y = Z] from \mydatatwelve;	 		  
		  
		  		  	  	
		\addplot [
		densely dotted,  
		 line width = 0.85 pt,
		 color = black
		 ]
		  table[y = Z] from \mydatafourteen;	 	

		\addplot [		
		densely dotted,  
		 line width = 0.85 pt,  
		 color = black
		 ]
		  table[y = Z] from \mydatafifteen;	 		  	  
		    
\end{axis}
\end{tikzpicture}
}
\centering
  \subfigure{
       \centering
\begin{tikzpicture}[scale=0.87]
\begin{axis}[
			legend style={at = {(0.98,0.555)} ,font=\tiny},		
			height = 37.0ex,
			width = 48.0ex,
			grid = major,
			tick pos=left, 
			xlabel shift = 0 pt,
			xmode=log,
			xminorticks = false,	
		    yminorticks = false,	
		    y tick label style={
        /pgf/number format/.cd,
            fixed,
        /tikz/.cd
    		}, 				
			ytick={0, 10, 20, 30, 40, 50, 60},
		xlabel= Number of Samples ={font=\normalsize},
			xmin = 1.0,
			xmax = 1000,
			ymin = 0,
			ymax = 60,
			yticklabel style = {font=\tiny,yshift=0.0ex},
            xticklabel style = {font=\tiny,xshift=0.0ex}			
			]

\pgfplotstableread{./Data/1S5Sigma40.txt}\mydataone
\pgfplotstableread{./Data/2S5Sigma40.txt}\mydatatwo
\pgfplotstableread{./Data/3S5Sigma40.txt}\mydatathree
\pgfplotstableread{./Data/4S5Sigma40.txt}\mydatafour
\pgfplotstableread{./Data/5S5Sigma40.txt}\mydatafive
\pgfplotstableread{./Data/6S5Sigma40.txt}\mydatafourteen
\pgfplotstableread{./Data/7S5Sigma40.txt}\mydatafifteen
\pgfplotstableread{./DataDual/1S5Sigma40.txt}\mydatasix
\pgfplotstableread{./DataDual/2S5Sigma40.txt}\mydataseven
\pgfplotstableread{./DataDual/3S5Sigma40.txt}\mydataeight
\pgfplotstableread{./DataDual/4S5Sigma40.txt}\mydatanine
\pgfplotstableread{./DataDual/5S5Sigma40.txt}\mydataten
\pgfplotstableread{./DataDual/6S5Sigma40.txt}\mydataeleven
\pgfplotstableread{./DataDual/7S5Sigma40.txt}\mydatatwelve
		
		\addplot [
		densely dotted,  
		 line width = 0.85 pt,
		 color = black
		]		
		 table[y = Z] from \mydataone;
		 
		 \addplot [
		densely dotted,  
		 line width = 0.85 pt,
		 color = black
		 ]
		  table[y = Z] from \mydatatwo;	 
		  
		  \addplot [
		densely dotted,  
		 line width = 0.85 pt,
		 color = black
		 ]
		  table[y = Z] from \mydatathree;	 

	  \addplot [
		densely dotted,  
		 line width = 0.85 pt,
		 color = black
		 ]
		  table[y = Z] from \mydatafour;	 
		  
	  \addplot [
		densely dotted,  
		 line width = 0.85 pt,
		 color = black
		 ]
		  table[y = Z] from \mydatafive;	 		  
		  
	  \addplot [
		 line width = 0.65 pt,
		 color = blue,
		 densely dashed
		 ]
		  table[y = Z] from \mydatasix;	 		  
		  
		  	  \addplot [
		 line width = 0.65 pt,
		 color = blue,
		 densely dashed
		 ]
		  table[y = Z] from \mydataseven;	 		  
		
			  	  \addplot [
		 line width = 0.65 pt,
		 color = blue,
		 densely dashed
		 ]
		  table[y = Z] from \mydataeight;	 		  
		  
		  	  	  \addplot [
		 line width = 0.65 pt,
		 color = blue,
		 densely dashed
		 ]
		  table[y = Z] from \mydatanine;	 		  
		  
		  	  	  \addplot [
		 line width = 0.65 pt,
		 color = blue,
		 densely dashed
		 ]
		  table[y = Z] from \mydataten;	 		  
		  
		  	  	  \addplot [
		 line width = 0.55 pt,
		 color = blue,
		 densely dashed
		 ]
		  table[y = Z] from \mydataeleven;	 	

		  	  	  \addplot [
		 line width = 0.65 pt,
		 color = blue,
		 densely dashed
		 ]
		  table[y = Z] from \mydatatwelve;	 		  
		  
		  		  	  	
		\addplot [
		densely dotted,  
		 line width = 0.85 pt,
		 color = black
		 ]
		  table[y = Z] from \mydatafourteen;	 	

		\addplot [		  
		densely dotted,  
		 line width = 0.85 pt,
		 color = black
		 ]
		  table[y = Z] from \mydatafifteen;	 		  	  
		    
%
%





\end{axis}
\end{tikzpicture}
}
\quad
  \subfigure{
       \centering
\begin{tikzpicture}[scale=0.87]
\begin{axis}[
			legend style={at = {(0.98,0.555)} ,font=\tiny},		
			height = 37.0ex,
			width = 48.0ex,
			grid = major,
			tick pos=left, 
			xlabel shift = 0 pt,
			xmode=log,
			xminorticks = false,	
		    yminorticks = false,	
		    y tick label style={
        /pgf/number format/.cd,
            fixed,
        /tikz/.cd
    		}, 				
			ytick={20, 21, 22, 23, 24, 25, 26, 27, 28, 29, 30},
			xtick={100, 1000},
		xlabel= Number of Samples ={font=\normalsize},
			xmin = 100,
			xmax = 1000,
			ymin = 20,
			ymax = 27,
			yticklabel style = {font=\tiny,yshift=0.0ex},
            xticklabel style = {font=\tiny,xshift=0.0ex}			
			]

\pgfplotstableread{./Data/1S5Sigma40.txt}\mydataone
\pgfplotstableread{./Data/2S5Sigma40.txt}\mydatatwo
\pgfplotstableread{./Data/3S5Sigma40.txt}\mydatathree
\pgfplotstableread{./Data/4S5Sigma40.txt}\mydatafour
\pgfplotstableread{./Data/5S5Sigma40.txt}\mydatafive
\pgfplotstableread{./Data/6S5Sigma40.txt}\mydatafourteen
\pgfplotstableread{./Data/7S5Sigma40.txt}\mydatafifteen
\pgfplotstableread{./DataDual/1S5Sigma40.txt}\mydatasix
\pgfplotstableread{./DataDual/2S5Sigma40.txt}\mydataseven
\pgfplotstableread{./DataDual/3S5Sigma40.txt}\mydataeight
\pgfplotstableread{./DataDual/4S5Sigma40.txt}\mydatanine
\pgfplotstableread{./DataDual/5S5Sigma40.txt}\mydataten
\pgfplotstableread{./DataDual/6S5Sigma40.txt}\mydataeleven
\pgfplotstableread{./DataDual/7S5Sigma40.txt}\mydatatwelve
		
		\addplot [
		densely dotted,  
		 line width = 0.85 pt,
		 color = black
		]		
		 table[y = Z] from \mydataone;
		 
		 \addplot [
		 densely dotted,  
		 line width = 0.85 pt,
		 color = black
		 ]
		  table[y = Z] from \mydatatwo;	 
		  
		  \addplot [
		 densely dotted,  
		 line width = 0.85 pt,
		 color = black
		 ]
		  table[y = Z] from \mydatathree;	 

	  \addplot [
		 densely dotted,  
		 line width = 0.85 pt,
		 color = black
		 ]
		  table[y = Z] from \mydatafour;	 
		  
	  \addplot [
		 densely dotted,  
		 line width = 0.85 pt,
		 color = black
		 ]
		  table[y = Z] from \mydatafive;	 		  
		  
	  \addplot [
		 line width = 0.65 pt,
		 color = blue,
		 densely dashed
		 ]
		  table[y = Z] from \mydatasix;	 		  
		  
		  	  \addplot [
		 line width = 0.65 pt,
		 color = blue,
		 densely dashed
		 ]
		  table[y = Z] from \mydataseven;	 		  
		
			  	  \addplot [
		 line width = 0.65 pt,
		 color = blue,
		 densely dashed
		 ]
		  table[y = Z] from \mydataeight;	 		  
		  
		  	  	  \addplot [
		 line width = 0.65 pt,
		 color = blue,
		 densely dashed
		 ]
		  table[y = Z] from \mydatanine;	 		  
		  
		  	  	  \addplot [
		 line width = 0.65 pt,
		 color = blue,
		 densely dashed
		 ]
		  table[y = Z] from \mydataten;	 		  
		  
		  	  	  \addplot [
		 line width = 0.65 pt,
		 color = blue,
		 densely dashed
		 ]
		  table[y = Z] from \mydataeleven;	 	

		  	  	  \addplot [
		 line width = 0.65 pt,
		 color = blue,
		 densely dashed
		 ]
		  table[y = Z] from \mydatatwelve;	 		  
		  
		  		  	  	
		\addplot [
		 densely dotted,  
		 line width = 0.85 pt,
		 color = black
		 ]
		  table[y = Z] from \mydatafourteen;	 	

		\addplot [		  
		 densely dotted,  
		 line width = 0.85 pt,
		 color = black
		 ]
		  table[y = Z] from \mydatafifteen;	 		  	  
		    
%
%





\end{axis}
\end{tikzpicture}
}
\vspace{0ex}
\caption{\label{fig:PlotsSigma5}%
Estimated marginal variances as a function of the number of samples $L$ for a $15\times 15$ Gaussian Markov random field 
with PDF as in~(\ref{eqn:ProbPGaussThin}), with 
periodic boundaries, and with $\sigma = 5$.
The right panel illustrates the zoomed 
plots of the left panel for $L \in [10^2, 10^3]$.
The dotted black paths and the dashed
blue paths show $\hat\sigma_{\text{p}}^2$ and $\hat\sigma_{\text{d} \rightarrow \text{p}}^2$ respectively. 
The intervariable variance is set to: (Top) $s = 1$, 
(Middle) $s = 20$, and (Bottom) $s = 40$.}
\end{figure}


\subsection{Numerical Experiments}

We consider an $N = 15\times 15$ Gaussian Markov random field with periodic boundary conditions and with 
PDF given by~(\ref{eqn:ProbPGaussThin}). In both domains (i.e., the primal and the dual domains), we employ the 
Gibbs sampling algorithm with a systematic (i.e., deterministic) 
sweep strategy~\cite[Chapter 10]{christian1999monte} to estimate the marginal 
densities.\footnote{Applying the Gibbs sampling algorithm to Gaussian Markov random fields is straightforward and can be 
represented in matrix form.  The algorithm can be viewed as the stochastic generalization of the Gauss-Seidel algorithm, as 
pointed out in~\citep[Section VIII]{goodman1989multigrid}.}

We set $\sigma = 5$ and 
the number of samples $L = 10^3$ in all the experiments. The exact value of $\sigma_{\text{p}}^2$ is computed by inverting the
information (i.e., precision) matrix associated with the model, which is feasible for this size of graph.
Notice that the total number of variables is $|\VV| = N = 225$ in the primal domain, and is $|\EE| = 2N = 450$ in the dual domain. 

In Fig.~(\ref{fig:PlotsSigma5}), the dotted 
black paths show $\hat\sigma_{\text{p}}^2$ the estimated marginal variance obtained directly in the primal NFG and 
the dashed blue paths show $\hat\sigma_{\text{d} \rightarrow \text{p}}^2$ the estimated marginal variance in the primal NFG 
obtained from $\hat\sigma_{\text{d}}^2$, and then mapped to the primal domain via~(\ref{eqn:TransVarGauss}). In all plots the horizontal axis shows the number of samples. 
In order to compare better the convergence, the right panel of Fig.~(\ref{fig:PlotsSigma5}) illustrates the zoomed 
plots of the left panel in the range $[10^2, 10^3]$.

In Fig.~(\ref{fig:PlotsSigma5})-Top, we set the intervariable variance $s = 1$. The exact value of $\sigma_{\text{p}}^2$ is about $0.5589$. The Gibbs sampling 
algorithm converges quickly in the primal NFG, where after $10^2$ samples the relative error is about $10^{-1}$. 
However, the Gibbs sampler suffers from slow and erratic convergence in the dual NFG. 
We set $s = 20$ in Fig.~(\ref{fig:PlotsSigma5})-Middle. In this example $\sigma_{\text{p}}^2$ is about $20.2046$. 
We observe that increasing $s$ improves the convergence in the dual domain, but degrades the convergence in the primal domain. 
In the last experiment, we set $s = 40$. 
In this case, the exact value of $\sigma_{\text{p}}^2$ is about $23.5498$. Simulation results are shown in Fig.~(\ref{fig:PlotsSigma5})-Bottom, 
where, in contrast to Fig.~(\ref{fig:PlotsSigma5})-Top, 
convergence in the primal domain is slow. However, the Gibbs sampling algorithm in the dual domain can provide very accurate 
estimates of the marginal variance. A relative error of about $1.5\times 10^{-3}$ is achieved using only $10^2$ samples. 

Intuitively, as $s^2$ grows, the generated samples in the primal domain may vary widely, which leads to unstable estimates of $\sigma_{\text{p}}^2$.
The situation is the opposite in the dual domain, as in the Fourier transform of (\ref{eqn:ProbPGauss2}) 
the standard deviation $s$ is replaced by $1/s$, cf.~(\ref{eqn:PdGauss2}). 
The potential advantages of performing the computations in the dual domain needs to be further investigated for other choices 
of parameters, e.g., by fixing $s^2$ and varying the vertex variance $\sigma^2$.

\section{Conclusion}

We proved that the edge/vertex marginals densities of a
primal NFG and the corresponding edge/vertex marginal densities of its dual NFG are related  
via local mappings. The mapping provides a simple procedure to transform 
simultaneously the estimated marginals 
from one domain to the other. Furthermore, the mapping relies on no assumptions
on the size or on the topology of the graphical model. 
Details of the mapping, including its
fixed points, were derived for the Ising model and were extended to
non-binary models (e.g., the Potts model) and to continuous models. 
We discussed that the subgraphs-world process can be employed as a rapidly mixing Markov
chain to generate configurations in the dual NFG of the ferromagnetic Ising models in a positive external field.
In this case, the edge marginals densities of the dual NFG can first be obtained in polynomial time in the dual domain, and then 
be transformed to the primal domain using the proposed mappings.

Our numerical experiments show that estimating the marginals densities in the dual NFG 
can sometimes significantly improve the quality of
approximations in terms of relative error. E.g., variational algorithms in the dual NFG of the ferromagnetic Ising and Potts models 
converge faster in the low-temperature 
regime (i.e., when the coupling parameters are large), and, for some settings, the Gibbs sampling algorithm in the dual domain 
can give more accurate estimates 
of the marginal densities of Gaussian Markov random fields with the thin-membrane prior.

Following Remark 4, the factors in the dual NFG of the Ising model can in general take negative values.
In principle, marginal functions can still be estimated via 
the BP algorithm in these cases.
However, we observed that in many such examples (e.g., spin glasses and antiferromagnetic Ising models) 
BP suffers from slow and erratic convergence in the dual domain. Analyzing the dynamics of the BP algorithm and designing new 
and improved variational algorithms in NFGs with
negative (or complex) factors is left for future work. 

\appendix

\section*{Appendix A: Factor Graphs and normal relizations}
\label{app:ApA}


\begin{figure}[t]
  \centering
  \begin{tikzpicture}[scale=0.98]
\draw[fill=black] (0.25,2.25) circle (1.0mm);
\draw[fill=black] (4.25,2.25) circle (1.0mm);
\draw[fill=black] (8.25,2.25) circle (1.0mm);
%
\draw [line width=0.22mm] (2, 2) rectangle (2.5,2.5);
\draw [line width=0.22mm] (6, 2) rectangle (6.5,2.5);
\draw [line width=0.22mm] (4, 1) rectangle (4.5,1.5);
\draw [line width=0.22mm] (0.25,2.25) -- (2, 2.25);
\draw [line width=0.22mm] (2.5,2.25) -- (4.25, 2.25);
\draw [line width=0.22mm] (4.25,2.25) -- (6, 2.25);
\draw [line width=0.22mm] (6.5,2.25) -- (8.25, 2.25);
\draw [line width=0.22mm] (0.25,2.25) -- (0.25, 3.25);
\draw [line width=0.22mm] (4.25,2.25) -- (4.25, 3.25);
\draw [line width=0.22mm] (8.25,2.25) -- (8.25, 3.25);
\draw [line width=0.22mm] (0.25,1.25) -- (4.0, 1.25);
\draw [line width=0.22mm] (4.5,1.25) -- (8.25, 1.25);
%
\draw [line width=0.22mm] (0.25,2.25) -- (0.25, 1.25);
\draw [line width=0.22mm] (8.25,2.25) -- (8.25, 1.25);
\draw [line width=0.22mm] (0.05, 3.25) rectangle (0.45,3.65);
\draw [line width=0.22mm] (4.05, 3.25) rectangle (4.45,3.65);
\draw [line width=0.22mm] (8.05, 3.25) rectangle (8.45,3.65);
%
%
 \draw (0.78,3.43) node{$\phi_1$};
  \draw (4.78,3.43) node{$\phi_2$};
 \draw (8.78,3.43) node{$\phi_3$};
 \draw (2.25,2.9) node{$\psi_1$};
 \draw (6.25,2.9) node{$\psi_2$};
  \draw (4.25,0.65) node{$\psi_3$};
  \draw (-0.1,2.65) node{$X_1$};
  \draw (3.92,2.65) node{$X_2$};
 \draw (7.92, 2.65) node{$X_3$};
  \end{tikzpicture}
  \caption{\label{fig:GridModAgainFactor}
The factor graph of (\ref{eqn:ProbPAppendixA1}).
The filled circles show the variable nodes and the empty boxes represent the factor nodes.
}
%
%
%
  \centering
  \begin{tikzpicture}[scale=0.98]
 \draw[fill=black] (0.25,2.25) circle (1.0mm);
\draw[fill=black] (4.25,2.25) circle (1.0mm);
\draw[fill=black] (8.25,2.25) circle (1.0mm);
\draw[fill=black] (2.25, 2.88) circle (1.0mm);
\draw[fill=black] (6.25, 2.88) circle (1.0mm);
\draw[fill=black] (4.25, 0.62) circle (1.0mm);
%
\draw [line width=0.22mm] (2, 2) rectangle (2.5,2.5);
\draw [line width=0.22mm] (6, 2) rectangle (6.5,2.5);
\draw [line width=0.22mm] (4, 1) rectangle (4.5,1.5);
\draw [line width=0.22mm] (0.25,2.25) -- (2, 2.25);
\draw [line width=0.22mm] (2.5,2.25) -- (4.25, 2.25);
\draw [line width=0.22mm] (4.25,2.25) -- (6, 2.25);
\draw [line width=0.22mm] (6.5,2.25) -- (8.25, 2.25);
\draw [line width=0.22mm] (0.25,2.25) -- (0.25, 3.25);
\draw [line width=0.22mm] (2.25,2.5) -- (2.25, 3.25);
\draw [line width=0.22mm] (4.25,2.25) -- (4.25, 3.25);
\draw [line width=0.22mm] (6.25,2.5) -- (6.25, 3.25);
\draw [line width=0.22mm] (8.25,2.25) -- (8.25, 3.25);
\draw [line width=0.22mm] (0.25,1.25) -- (4.0, 1.25);
\draw [line width=0.22mm] (4.5,1.25) -- (8.25, 1.25);
\draw [line width=0.22mm] (4.25,1.0) -- (4.25, 0.25);
\draw [line width=0.22mm] (0.25,2.25) -- (0.25, 1.25);
\draw [line width=0.22mm] (8.25,2.25) -- (8.25, 1.25);
\draw [line width=0.22mm] (0.05, 3.25) rectangle (0.45,3.65);
\draw [line width=0.22mm] (2, 3.25) rectangle (2.5,3.75);
\draw [line width=0.22mm] (4.05, 3.25) rectangle (4.45,3.65);
\draw [line width=0.22mm] (6, 3.25) rectangle (6.5,3.75);
\draw [line width=0.22mm] (8.05, 3.25) rectangle (8.45,3.65);
\draw [line width=0.22mm] (4, 0.25) rectangle (4.5,-0.25);
\node at (2.25, 2.25){$+$};
\node at (6.25, 2.25){$+$};
\node at (4.25, 1.25){$+$};
 \draw (0.78,3.43) node{$\phi_1$};
  \draw (4.78,3.43) node{$\phi_2$};
 \draw (8.78,3.43) node{$\phi_3$};
 \draw (2.78,3.52) node{$\psi_1$};
 \draw (6.78,3.52) node{$\psi_2$};
  \draw (3.75,0.0) node{$\psi_3$};
  \draw (-0.1,2.65) node{$X_1$};
  \draw (3.92,2.65) node{$X_2$};
 \draw (7.92, 2.65) node{$X_3$}; 
 \draw (1.9,2.86) node{$Y_1$};
  \draw (5.9,2.86) node{$Y_2$};
    \draw (4.65,0.61) node{$Y_3$};
  \end{tikzpicture}
  \caption{\label{fig:GridModAgainFactor2}
The factor graph of (\ref{eqn:ProbPAppendixA2}). The filled circles show the variable nodes and the empty 
boxes represent the factor nodes. The boxes labeled ``$+$'' are instances of zero-sum indicator 
factors given by (\ref{eqn:Definition2}).
}
\end{figure}


\begin{figure}[t]
  \centering
  \begin{tikzpicture}[scale=0.98]

\draw [line width=0.22mm] (0, 2) rectangle (0.5,2.5);
\draw [line width=0.22mm] (2, 2) rectangle (2.5,2.5);
\draw [line width=0.22mm] (4, 2) rectangle (4.5,2.5);
\draw [line width=0.22mm] (6, 2) rectangle (6.5,2.5);
\draw [line width=0.22mm] (8, 2) rectangle (8.5,2.5);
\draw [line width=0.22mm] (4, 1) rectangle (4.5,1.5);
\draw [line width=0.22mm] (0.5,2.25) -- (2, 2.25);
\draw [line width=0.22mm] (2.5,2.25) -- (4, 2.25);
\draw [line width=0.22mm] (4.5,2.25) -- (6, 2.25);
\draw [line width=0.22mm] (6.5,2.25) -- (8, 2.25);
\draw [line width=0.22mm] (0.25,2.5) -- (0.25, 3.25);
\draw [line width=0.22mm] (2.25,2.5) -- (2.25, 3.25);
\draw [line width=0.22mm] (4.25,2.5) -- (4.25, 3.25);
\draw [line width=0.22mm] (6.25,2.5) -- (6.25, 3.25);
\draw [line width=0.22mm] (8.25,2.5) -- (8.25, 3.25);
\draw [line width=0.22mm] (0.25,1.25) -- (4.0, 1.25);
\draw [line width=0.22mm] (4.5,1.25) -- (8.25, 1.25);
\draw [line width=0.22mm] (4.25,1.0) -- (4.25, 0.25);
\draw [line width=0.22mm] (0.25,2) -- (0.25, 1.25);
\draw [line width=0.22mm] (8.25,2) -- (8.25, 1.25);
\draw [line width=0.22mm] (0.05, 3.25) rectangle (0.45,3.65);
\draw [line width=0.22mm] (2, 3.25) rectangle (2.5,3.75);
\draw [line width=0.22mm] (4.05, 3.25) rectangle (4.45,3.65);
\draw [line width=0.22mm] (6, 3.25) rectangle (6.5,3.75);
\draw [line width=0.22mm] (8.05, 3.25) rectangle (8.45,3.65);
\draw [line width=0.22mm] (4, 0.25) rectangle (4.5,-0.25);
\node at (0.25, 2.20){$=$};
\node at (2.25, 2.25){$+$};
\node at (4.25, 2.20){$=$};
\node at (6.25, 2.25){$+$};
\node at (8.25, 2.20){$=$};
\node at (4.25, 1.25){$+$};
 \draw (0.78,3.43) node{$\phi_1$};
  \draw (4.78,3.43) node{$\phi_2$};
 \draw (8.78,3.43) node{$\phi_3$};
 \draw (2.78,3.52) node{$\psi_1$};
 \draw (6.78,3.52) node{$\psi_2$};
  \draw (3.75,0.0) node{$\psi_3$};
  \draw (0.0,2.86) node{$X_1$};
  \draw (1.2,1.98) node{$X^{\prime}_1$};
    \draw (-0.05,1.62) node{$X^{\prime\prime}_1$};
      \draw (3.25,1.98) node{$X^{\prime}_2$};
  \draw (3.98,2.86) node{$X_2$};
 \draw (7.98,2.86) node{$X_3$};
 \draw (2.0,2.86) node{$Y_1$};
  \draw (6.0,2.86) node{$Y_2$};
    \draw (4.5,0.61) node{$Y_3$};
  \end{tikzpicture}
  \caption{\label{fig:GridModAgain3}
The NFG of (\ref{eqn:ProbPAppendixA2}).
The unlabeled boxes represent the factors $\phi(\cdot)$ and $\psi(\cdot)$. Boxes labeled ``$=$'' are instances of equality 
indicator factors given by (\ref{eqn:Definition1}), and
the boxes labeled ``$+$'' are instances of zero-sum indicator 
factors as in (\ref{eqn:Definition2})}
%
%
  \centering
  \begin{tikzpicture}[scale=0.98]

\draw [line width=0.22mm] (0, 2) rectangle (0.5,2.5);
\draw [line width=0.22mm] (2, 2) rectangle (2.5,2.5);
\draw [line width=0.22mm] (4, 2) rectangle (4.5,2.5);
\draw [line width=0.22mm] (6, 2) rectangle (6.5,2.5);
\draw [line width=0.22mm] (8, 2) rectangle (8.5,2.5);
\draw [line width=0.22mm] (4, 1) rectangle (4.5,1.5);
\draw [line width=0.22mm] (0.5,2.25) -- (2, 2.25);
\draw [line width=0.22mm] (2.5,2.25) -- (4, 2.25);
\draw [line width=0.22mm] (4.5,2.25) -- (6, 2.25);
\draw [line width=0.22mm] (6.5,2.25) -- (8, 2.25);
\draw [line width=0.22mm] (0.25,2.5) -- (0.25, 3.25);
\draw [line width=0.22mm] (2.25,2.5) -- (2.25, 3.25);
\draw [line width=0.22mm] (4.25,2.5) -- (4.25, 3.25);
\draw [line width=0.22mm] (6.25,2.5) -- (6.25, 3.25);
\draw [line width=0.22mm] (8.25,2.5) -- (8.25, 3.25);
\draw [line width=0.22mm] (0.25,1.25) -- (4.0, 1.25);
\draw [line width=0.22mm] (4.5,1.25) -- (8.25, 1.25);
\draw [line width=0.22mm] (4.25,1.0) -- (4.25, 0.25);
\draw [line width=0.22mm] (0.25,2) -- (0.25, 1.25);
\draw [line width=0.22mm] (8.25,2) -- (8.25, 1.25);
\draw [line width=0.22mm] (0.05, 3.25) rectangle (0.45,3.65);
\draw [line width=0.22mm] (2, 3.25) rectangle (2.5,3.75);
\draw [line width=0.22mm] (4.05, 3.25) rectangle (4.45,3.65);
\draw [line width=0.22mm] (6, 3.25) rectangle (6.5,3.75);
\draw [line width=0.22mm] (8.05, 3.25) rectangle (8.45,3.65);
\draw [line width=0.22mm] (4, 0.25) rectangle (4.5,-0.25);
\node at (0.25, 2.25){$+$};
\node at (2.25, 2.20){$=$};
\node at (4.25, 2.25){$+$};
\node at (6.25, 2.20){$=$};
\node at (8.25, 2.25){$+$};
\node at (4.25, 1.20){$=$};
 \draw (0.78,3.43) node{$\tilde \phi_1$};
  \draw (4.78,3.43) node{$\tilde \phi_2$};
 \draw (8.78,3.43) node{$\tilde \phi_3$};
 \draw (2.78,3.52) node{$\tilde \psi_1$};
 \draw (6.78,3.52) node{$\tilde \psi_2$};
  \draw (3.75,0.0) node{$\tilde \psi_3$};
  \draw (0.0,2.86) node{$\tilde X_1$};
  \draw (3.98,2.86) node{$\tilde X_2$};
 \draw (7.98,2.86) node{$\tilde X_3$};
 \draw (2.0,2.86) node{$\tilde Y_1$};
  \draw (6.0,2.86) node{$\tilde Y_2$};
    \draw (4.5,0.61) node{$\tilde Y_3$};
  \end{tikzpicture}
  \caption{\label{fig:GridModAgain4}
The dual of the NFG in Fig.~\ref{fig:GridModAgain3}.}
\end{figure}


A factor graph is a diagram that represents the factorization of a multivariate function, in which there is a variable node for
each variable and a factor node for each factor. An edge connects the variable node of $x_i$ to the factor node of $\psi_j(\cdot)$ if 
and only if $x_i$ is an argument of $\psi_j(\cdot)$. See~\citep{KFL:01}, for more details.

As an example, let $\calA = \ZZ/2\ZZ$ and assume that $\mathrm{\pi}_\text{p}(\cdot)$ can be factored as
\begin{equation} 
\label{eqn:ProbPAppendixA1}
\mathrm{\pi}_\text{p} (x_1, x_2, x_3) \propto \phi_1(x_1)\phi_2(x_2)\phi_3(x_3) \psi_1(x_1, x_2)\psi_1(x_2, x_3)\psi_1(x_1, x_3)
\end{equation}
up to scale. The factor graph in Fig.~(\ref{fig:GridModAgainFactor}) expresses the factorization in~(\ref{eqn:ProbPAppendixA1}), where the filled
circles show the variable nodes and the empty boxes represent the factor nodes.

If we further assumed that $\psi_{1}(\cdot)$, $\psi_{2}(\cdot)$, and $\psi_{3}(\cdot)$ are only functions of the difference between their 
two arguments, we can express $\mathrm{\pi}_\text{p} (\cdot)$ by the following factorization 
\begin{equation} 
\label{eqn:ProbPAppendixA2}
\mathrm{\pi}_\text{p} (x_1, x_2, x_3) \propto \phi_1(x_1)\phi_2(x_2)\phi_3(x_3) \psi_1(y_1)\psi_1(y_2)\psi_1(y_3),
\end{equation}
where $y_1 = x_1 - x_2$, $y_2 = x_2 - x_3$, and $y_3 = x_3 - x_1$.

The factor graph that corresponds to (\ref{eqn:ProbPAppendixA2}) is shown in Fig.~(\ref{fig:GridModAgainFactor2}).
The boxes labeled ``$+$'' are instances of zero-sum indicator 
factors given by (\ref{eqn:Definition2}), which impose the
constraint that all their incident variables sum 
to zero modulo two. 

By contrast, in NFGs, introduced in~\citep{Forney:01}, variable are represented by edges.
The NFG that corresponds to (\ref{eqn:ProbPAppendixA2}) is illustrated in Fig.~(\ref{fig:GridModAgain3}). Since
variables $y_1$, $y_2$, and $y_3$ appear in only two factors, they are easily represented by their corresponding edges. However, variables 
$x_1$, $x_2$, and $x_3$ that appear in more than two factors are replicated via equality indicator factors.
E.g., in Fig.~(\ref{fig:GridModAgain3})
\begin{equation} 
\delta_{=}(x_1, x_1^\prime, x_1^{\prime\prime}) =  \left\{ \begin{array}{ll}
    1, & \text{if $x_1 = x_1^\prime = x_1^{\prime\prime}$} \\
    0, & \text{otherwise,}
  \end{array} \right.
\end{equation}
which can also be expressed as $\delta_{=}(x_1, x_1^\prime, x_1^{\prime\prime}) = \delta(x_1 - x_1^\prime)\delta(x_1 - x_1^{\prime\prime})$. 
Here $\delta(\cdot)$ denotes either the Kronecker delta function (in discrete models discussed in 
Sections~\ref{sec:PrimalNFG} and~\ref{sec:Dual}) or 
the Dirac delta function (in continuous models discussed in Section~\ref{sec:Continuous}).
Replicating a variable by an equality indicator factor allows us to enforce the condition that each variable appears in (at most) two 
factors. As a result, in NFGs there is an edge for every variable~\citep{Lg:ifg2004, Forney:11}. 

The dual NFG can be obtained by replacing each variable by its corresponding dual variable and 
each factor by its 1D DFT, cf. Section~\ref{sec:Dual}. The dual of the NFG in Fig.~\ref{fig:GridModAgain3} is
illustrated in Fig.~\ref{fig:GridModAgain4}.


\section*{Appendix B: Marginal densities of the 1D Ising and Potts models}
\label{app:ApB}

The 1D Ising and Potts models with free or with periodic boundary conditions are exactly solved models. 
It is therefore possible to compute the edge marginal probabilities exactly in these models via the sum-product 
algorithm~\citep{KFL:01} or via the transfer matrix method~\citep{Baxter:07}. 
Here we employ our proposed mappings as an alternative (and somewhat simpler) way to compute 
such marginals.

The Primal NFG of the 1D Ising model in zero field and with periodic boundaries is shown in Fig.~\ref{fig:1DPottsFreePrimal}, 
where the unlabeled boxes represent~(\ref{eqn:IsingPot1}). To construct the dual NFG of the model, variables are replaced by their 
corresponding dual variable, factors (\ref{eqn:IsingPot1}) are replaced by (\ref{eqn:IsingDJ}), equality indicator factors 
are replaced by zero-sum indicator factors and vice-versa. The ``$\circ$'' symbols are immaterial as $\calA = \ZZ/2\ZZ$.
Fig.~\ref{fig:1DPottsFreeDual} shows the
dual of Fig.~\ref{fig:1DPottsFreePrimal}.

In the dual NFG $\tilde X_1 = \tilde X_2 = \ldots = \tilde X_{|\EE|}$, therefore there are only two
valid configurations, namely the all-zeros and the all-ones configurations. Hence
\begin{equation}
Z_{\text{d}} = 2^{|\EE|}\big(\prod_{e \in \EE}\cosh(\beta J_e) + \prod_{e \in \EE}\sinh(\beta J_e)\big),
\end{equation}
which gives
\begin{align}
\pi_{\text{d},e}(0) & = \myfrac{\prod_{e \in \EE}\cosh(\beta J_e)}{\prod_{e \in \EE}\cosh(\beta J_e) + \prod_{e \in \EE}\sinh(\beta J_e)} \\
& = \myfrac{1}{1 + \prod_{e \in \EE}\tanh(\beta J_e)}
\end{align}
and
\begin{equation}
\pi_{\text{d},e}(1) = \myfrac{\prod_{e \in \EE}\tanh(\beta J_e)}{1 + \prod_{e \in \EE}\tanh(\beta J_e)}
\end{equation}

From (\ref{eqn:MapDP}), we obtain the edge marginal probabilities of the primal NFG as
\begin{align}
\pi_{\text{p},e}(0) & = \myfrac{\textrm{e}^{\beta J_e}}{2\cosh(\beta J_e)}\myfrac{1+ 
\prod_{\e \in \EE \setminus \{e\}}\tanh(\beta J_\e)}{1 + \prod_{e \in \EE}\tanh(\beta J_e)} \\
\pi_{\text{p},e}(1) & = \myfrac{\textrm{e}^{-\beta J_e}}{2\cosh(\beta J_e)}\myfrac{1- 
\prod_{\e \in \EE \setminus \{e\}}\tanh(\beta J_\e)}{1 + \prod_{e \in \EE}\tanh(\beta J_e)}
\end{align}

In the dual NFG of the 1D Ising model with free boundary conditions, there is only one valid configuration, namely the all-zeros configuration.
Consequently
\begin{equation}
\pi_{\text{d},e}(0) = 1-\pi_{\text{d},e}(1) = 1
\end{equation}
and therefore from (\ref{eqn:MapDP}) we obtain
\begin{align}
\pi_{\text{p},e}(0) & = \myfrac{\textrm{e}^{\beta J_e}}{2\cosh(\beta J_e)} \label{eqn:MargIsing1Dfree} \\
\pi_{\text{p},e}(1) & =  \myfrac{\textrm{e}^{-\beta J_e}}{2\cosh(\beta J_e)}
\end{align}
for $e \in \EE$. Due to symmetry 
\begin{equation}
\pi_{\text{p},v}(0) = \frac{1}{2}\pi_{\text{p},e}(0) + \frac{1}{2}\pi_{\text{p},e}(1)
\end{equation}
for $v \in \VV$. Thus
\begin{equation}
\pi_{\text{p},v}(0) = \pi_{\text{p},v}(1) = \frac{1}{2}
\end{equation}

The marginal probability $\pi_{\text{p},e}(0)$ in (\ref{eqn:MargIsing1Dfree}) attains the lower bound in (\ref{eqn:BoundonPe0}).
We also observe that in the thermodynamic limit our results are independent of the boundary conditions.  
In other words, the effect of the boundary conditions becomes negligible as $|\VV|, |\EE| \to \infty$.


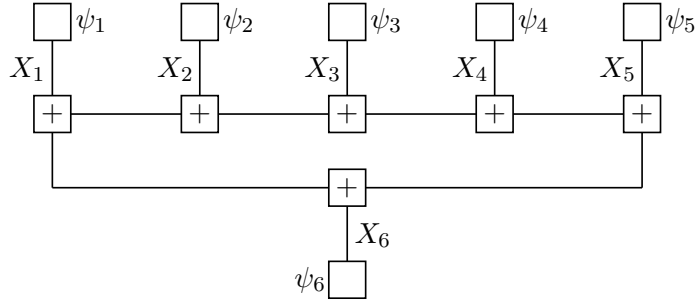
\begin{figure}[t]
  \centering
  \begin{tikzpicture}[scale=0.98]

\draw [line width=0.22mm] (0, 2) rectangle (0.5,2.5);
\draw [line width=0.22mm] (2, 2) rectangle (2.5,2.5);
\draw [line width=0.22mm] (4, 2) rectangle (4.5,2.5);
\draw [line width=0.22mm] (6, 2) rectangle (6.5,2.5);
\draw [line width=0.22mm] (8, 2) rectangle (8.5,2.5);
\draw [line width=0.22mm] (4, 1) rectangle (4.5,1.5);
\draw [line width=0.22mm] (0.5,2.25) -- (2, 2.25);
\draw [line width=0.22mm] (2.5,2.25) -- (4, 2.25);
\draw [line width=0.22mm] (4.5,2.25) -- (6, 2.25);
\draw [line width=0.22mm] (6.5,2.25) -- (8, 2.25);
\draw [line width=0.22mm] (0.25,2.5) -- (0.25, 3.25);
\draw [line width=0.22mm] (2.25,2.5) -- (2.25, 3.25);
\draw [line width=0.22mm] (4.25,2.5) -- (4.25, 3.25);
\draw [line width=0.22mm] (6.25,2.5) -- (6.25, 3.25);
\draw [line width=0.22mm] (8.25,2.5) -- (8.25, 3.25);
\draw [line width=0.22mm] (0.25,1.25) -- (4.0, 1.25);
\draw [line width=0.22mm] (4.5,1.25) -- (8.25, 1.25);
\draw [line width=0.22mm] (4.25,1.0) -- (4.25, 0.25);

\draw [line width=0.22mm] (0.25,2) -- (0.25, 1.25);
\draw [line width=0.22mm] (8.25,2) -- (8.25, 1.25);
\draw [line width=0.22mm] (0, 3.25) rectangle (0.5,3.75);
z\draw [line width=0.22mm] (2, 3.25) rectangle (2.5,3.75);
\draw [line width=0.22mm] (4.0, 3.25) rectangle (4.5,3.75);
\draw [line width=0.22mm] (6, 3.25) rectangle (6.5,3.75);
\draw [line width=0.22mm] (8.0, 3.25) rectangle (8.5,3.75);
\draw [line width=0.22mm] (4, 0.25) rectangle (4.5,-0.25);
\node at (0.25, 2.25){$+$};
\node at (2.25, 2.25){$+$};
\node at (4.25, 2.25){$+$};
\node at (4.25, 1.25){$+$};
\node at (6.25, 2.25){$+$};
\node at (8.25, 2.25){$+$};
 \draw (0.78,3.52) node{$\psi_1$};
 \draw (2.78,3.52) node{$\psi_2$};
  \draw (4.78,3.52) node{$\psi_3$};
 \draw (6.78,3.52) node{$\psi_4$};
  \draw (8.78,3.52) node{$\psi_5$};
  \draw (3.75,0.0) node{$\psi_6$};
  \draw (-0.09,2.86) node{$X_1$};
  \draw (1.92,2.86) node{$X_2$};
  \draw (3.92,2.86) node{$X_3$};
 \draw (5.92,2.86) node{$X_4$};
  \draw (7.92,2.86) node{$X_5$};
    \draw (4.6,0.61) node{$X_6$};
  \end{tikzpicture}
  \caption{\label{fig:1DPottsFreePrimal}
The primal NFG of the 1D Ising model in the absence of an external field and with periodic boundary conditions.
The unlabeled boxes represent~(\ref{eqn:IsingPot1}), boxes labeled ``$=$'' are instances of equality 
indicator factors as in (\ref{eqn:Definition1}),
the boxes labeled ``$+$'' are instances of zero-sum indicator 
factors given by (\ref{eqn:Definition2}).}
\end{figure}


\begin{figure}[t]
  \centering
  \begin{tikzpicture}[scale=0.98]

\draw [line width=0.22mm] (0, 2) rectangle (0.5,2.5);
\draw [line width=0.22mm] (2, 2) rectangle (2.5,2.5);
\draw [line width=0.22mm] (4, 2) rectangle (4.5,2.5);
\draw [line width=0.22mm] (6, 2) rectangle (6.5,2.5);
\draw [line width=0.22mm] (8, 2) rectangle (8.5,2.5);
\draw [line width=0.22mm] (4, 1) rectangle (4.5,1.5);
\draw [line width=0.22mm] (0.5,2.25) -- (2, 2.25);
\draw [line width=0.22mm] (2.5,2.25) -- (4, 2.25);
\draw [line width=0.22mm] (4.5,2.25) -- (6, 2.25);
\draw [line width=0.22mm] (6.5,2.25) -- (8, 2.25);
\draw [line width=0.22mm] (0.25,2.5) -- (0.25, 3.25);
\draw [line width=0.22mm] (2.25,2.5) -- (2.25, 3.25);
\draw [line width=0.22mm] (4.25,2.5) -- (4.25, 3.25);
\draw [line width=0.22mm] (6.25,2.5) -- (6.25, 3.25);
\draw [line width=0.22mm] (8.25,2.5) -- (8.25, 3.25);
\draw [line width=0.22mm] (0.25,1.25) -- (4.0, 1.25);
\draw [line width=0.22mm] (4.5,1.25) -- (8.25, 1.25);
\draw [line width=0.22mm] (4.25,1.0) -- (4.25, 0.25);

\draw [line width=0.22mm] (0.25,2) -- (0.25, 1.25);
\draw [line width=0.22mm] (8.25,2) -- (8.25, 1.25);
\draw [line width=0.22mm] (0, 3.25) rectangle (0.5,3.75);
\draw [line width=0.22mm] (2, 3.25) rectangle (2.5,3.75);
\draw [line width=0.22mm] (4.0, 3.25) rectangle (4.5,3.75);
\draw [line width=0.22mm] (6, 3.25) rectangle (6.5,3.75);
\draw [line width=0.22mm] (8.0, 3.25) rectangle (8.5,3.75);
\draw [line width=0.22mm] (4, 0.25) rectangle (4.5,-0.25);
\node at (0.25, 2.2){$=$};
\node at (2.25, 2.2){$=$};
\node at (4.25, 2.2){$=$};
\node at (4.25, 1.2){$=$};
\node at (6.25, 2.2){$=$};
\node at (8.25, 2.2){$=$};
 \draw (0.78,3.52) node{$\tilde\psi_1$};
 \draw (2.78,3.52) node{$\tilde\psi_2$};
  \draw (4.78,3.52) node{$\tilde\psi_3$};
 \draw (6.78,3.52) node{$\tilde \psi_4$};
  \draw (8.78,3.52) node{$\tilde \psi_5$};
  \draw (3.73,0.0) node{$\tilde \psi_6$};
  \draw (-0.09,2.86) node{$\tilde X_1$};
  \draw (1.92,2.86) node{$\tilde X_2$};
  \draw (3.92,2.86) node{$\tilde X_3$};
 \draw (5.92,2.86) node{$\tilde X_4$};
  \draw (7.92,2.86) node{$\tilde X_5$};
    \draw (4.6,0.61) node{$\tilde X_6$};
  \end{tikzpicture}
  \caption{\label{fig:1DPottsFreeDual}
The dual of the NFG in Fig.~\ref{fig:1DPottsFreePrimal}, where the unlabeled boxes represent~(\ref{eqn:IsingDJ}).}
\end{figure}
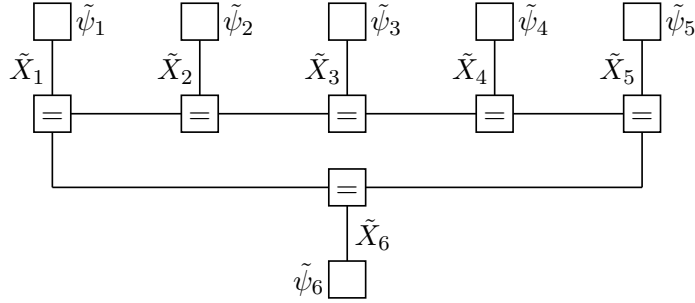


Finally, we briefly discuss similar mappings for the edge marginal probabilities of the 1D homogeneous $q$-state Potts model 
with periodic boundaries.

In the dual NFG of the model, the unlabeled boxes represent (\ref{eqn:PottsDJ}). 
There are exactly $q$
configurations in the dual domain, each with $\tilde X_1 = \tilde X_2 = \ldots = \tilde X_{|\EE|}$. 
Hence
\begin{equation}
Z_{\text{d}} = (\textrm{e}^{\beta J} + q -1)^{|\EE|} + (q-1)(\textrm{e}^{\beta J} -1)^{|\EE|},
\end{equation}
where the first term is the contribution of the all-zeros configuration, and the contribution of each of the 
remaining $q-1$ configurations to the partition function is $(\textrm{e}^{\beta J} -1)^{|\EE|}$.

Hence
\begin{equation}
\pi_{\text{d},e}(0)  = \myfrac{(\textrm{e}^{\beta J} + q -1)^{|\EE|}}{(\textrm{e}^{\beta J} + q -1)^{|\EE|} + (q-1)(\textrm{e}^{\beta J} -1)^{|\EE|}}
\end{equation}
and for $t \in \{1,2, \ldots, q-1\}$
\begin{equation}
\pi_{\text{d},e}(t) = \myfrac{(\textrm{e}^{\beta J} -1)^{|\EE|}}{(\textrm{e}^{\beta J} + q -1)^{|\EE|} + (q-1)(\textrm{e}^{\beta J} -1)^{|\EE|}}
\end{equation}

From (\ref{eqn:mappingPottsvector}) we obtain
\begin{equation}
\pi_{\text{p},e}(0)  = \textrm{e}^{\beta J}\myfrac{(\textrm{e}^{\beta J} + q -1)^{|\EE|-1}+(q-1)(\textrm{e}^{\beta J} -1)^{|\EE|-1}}{(\textrm{e}^{\beta J} + q -1)^{|\EE|} + (q-1)(\textrm{e}^{\beta J} -1)^{|\EE|}}
\end{equation}
The other entries of  $(\pi_{\text{p}, e}(a), a \in \calA)$ can be computed analogously.

\acks{The author is extremely grateful to G. David Forney, Jr.,
for his comments and for his continued support. The
author wishes to thank Justin Dauwels for his comments on an earlier draft of this paper. The author would like to thanks the editor, Mohammad Emtiyaz Khan, and 
the anonymous reviewers for their helpful feedback. }

\bibliography{mybib}

\end{document}